\documentclass[twoside,11pt]{article}

\usepackage{jmlr2e}

\usepackage{url}
\usepackage{xcolor}
\usepackage{amsmath}
\usepackage{amssymb}
\usepackage[binary-units=true]{siun itx}
\usepackage{mathtools}
\usepackage{amsfonts}
\usepackage{mathabx}
\usepackage{imakeidx}
\usepackage{marginnote}
\usepackage{verbatim}

\mathtoolsset{showonlyrefs=true}

\usepackage{hyperref}    

\usepackage{algorithm}
\usepackage[noend]{algpseudocode}

\usepackage{pdfsync}

\usepackage{tikz}
\usetikzlibrary{shapes,snakes}
\usetikzlibrary {positioning}
\usepackage{subcaption}
\usepackage{xr}

\newcommand{\virg}[1]{``#1'' }

\newcommand{\norm}[1]{\left\lVert#1\right\rVert}

  {\left\lbrace\begin{array}{@{}l@{}}}%
  {\end{array}\right.}

\usepackage{lastpage}
\jmlrheading{23}{2022}{1-\pageref{LastPage}}{4/19; Revised 10/21}{1/22}{19-267}{Lorenzo Rimella and Nick Whiteley} 

\ShortHeadings{Localization in Factorial hidden Markov models}{Rimella and Whiteley}
\firstpageno{1}

\begin{document}

\title{Exploiting locality in high-dimensional Factorial hidden Markov models}

\author{\name Lorenzo Rimella \email l.rimella@lancaster.ac.uk \\
       \addr Department of Mathematics and Statistics\\
       Lancaster University\\
       Lancaster, LA1 4YF, UK
       \AND
       \name Nick Whiteley \email nick.whiteley@bristol.ac.uk \\
       \addr Institute for Statistical Science\\
       School of Mathematics\\
       University of Bristol\\
       Bristol, BS8 1TW, UK\\
       and the Alan Turing Institute, UK}

\editor{Barbara Engelhardt}

\maketitle

\begin{abstract}%
We propose algorithms for approximate filtering and smoothing in high-dimensional Factorial hidden Markov models. The approximation involves discarding, in a principled way, likelihood factors according to a notion of locality in a factor graph associated with the emission distribution. This allows the exponential-in-dimension cost of exact filtering and smoothing to be avoided. We prove that the approximation accuracy, measured in a local total variation norm, is \virg{dimension-free} in the sense that as the overall dimension of the model increases the error bounds we derive do not necessarily degrade. A key step in the analysis is to quantify the error introduced by localizing the likelihood function in a Bayes' rule update. The factorial structure of the likelihood function which we exploit arises naturally when data have known spatial or network structure. We demonstrate the new algorithms on synthetic examples and a London Underground passenger flow problem, where the factor graph is effectively given by the train network.
\end{abstract}


\section{Introduction}\label{sec:intro}



Since early appearance in the statistical literature \citep{baum1966statistical,baum1970maximization} and popularization in speech recognition \citep{rabiner1989tutorial}, Hidden Markov models (HMMs) have been used to solve a broad range of problems ranging from texture recognition \citep{bose1994connected}, to gene prediction \citep{stanke2003gene}, and weather forecasting \citep{hughes1999non}. The influential paper of \cite{Ghahramani1997} introduced the class of Factorial hidden Markov models (FHMMs), in which the hidden Markov chain is a multivariate process, with a-priori independent coordinates. This structure provides a rich modelling framework to capture complex statistical patterns in data sequences and it is particularly suited to applications with straightforward local dependencies: traffic modelling, weather forecasting, etc.

The main objective of the present paper is to develop rigorous insight into how inference in FHMMs can be scaled up to high-dimensional problems. The problem one immediately faces when scaling up exact inference techniques for FHMM's is the computational cost of the matrix-vector operations underlying the well known forward-backward algorithm, which computes conditional distributions over hidden states given data, and in turn the Baum-Welch algorithm to perform maximum-likelihood estimation of static parameters. This cost typically grows exponentially with the dimension of the underlying state-space.  \cite{Ghahramani1997} derived a variant of the forward-backward algorithm which achieves a degree of efficiency by exploiting the structure of FHMMs, but the exponential-in-dimension scaling of cost cannot be avoided.

As an alternative to exact inference, \cite{Ghahramani1997} also proposed two families of variational methods for FHMMs, which allow approximate solution of the {smoothing} problem: computing conditional distributions over hidden states given past and future data. In turn, this allows approximate maximum likelihood parameter estimation via an expectation-maximization (EM) algorithm. The attractive feature of these variational approximations is,  in typical FHMMs, that their computational complexity is a low-order polynomial in state dimension. However, little seems to be known about their performance in the context of FHMMs from a theoretical point of view, other than the fact that by construction they minimize a Kullback-Liebler divergence criterion. As surveyed recently by \cite{blei2017variational}(Section 5), there are a number of strands of research into theoretical properties of variational methods for some classes of statistical models, such as convergence analysis for mixture models \citep{wang2006convergence}, and consistency studies for stochastic block models \citep{celisse2012consistency,bickel2013asymptotic}. However analysis specifically for FHMMs appears to be lacking, and to the authors' knowledge there are currently no detailed mathematical studies of how variational approximation errors for FHMMs scale with dimension, data record length, model parameters, etc.

\cite{boyen1998tractable, boyen1999exploiting}, proposed and studied inference methods in Dynamic Bayesian Networks which involve recursively approximating belief-state distributions by the product of their marginals and then propagating the result to the next time step. Particle filtering algorithms in the same vein appeared in \cite{ng2002factored}, \cite{brandao2006subspace} and \cite{besada2009parallel}. Ground-breaking theoretical work of \cite{rebeschini2015can} proved that similar algorithms can be used to conduct particle filtering efficiently in high-dimensions, using techniques based on the Dobrushin Comparison Theorem \citep[see for instance ][]{georgii2011gibbs}. Subsequently, \cite{rebeschini2014comparison} refined their analysis through generalized Dobrushin comparison theorems. \cite{finke2017approximate} extended these ideas from particle filtering to particle smoothing and studied dimension-independence of the asymptotic Monte Carlo variance. In the present work we do not consider sampling-based approximate inference techniques because we focus on FHMMs with discrete hidden states. If one were to consider FHMMs with continuous hidden states then following the blueprint of the aforementioned works one could devise particle filtering and smoothing versions of the approximate inference algorithms we propose, but this is beyond the scope of the present work. Another key difference between the aforementioned works and ours is that they do not considered FHMMs, but rather models in which posterior dependence across dimensions arises from the prior. For FHMMs, this dependence arises only from the emission distribution likelihood function and we use not only factorization but also the key notion of \emph{localization}, introduced below, to deal with this.

\subsection{Setting and Contributions}\label{sec:setting_and_cont}
In this paper we propose and study approximate inference algorithms for FHMMs called the \emph{Graph Filter} and \emph{Graph Smoother}.  In some ways, the Graph Filter and Smoother are similar in spirit to variational methods: they involve constructing approximate posterior distributions over hidden states which factorize across dimension. However, unlike variational methods, they do not involve explicit minimization of a Kullback-Liebler divergence criterion and therefore avoid the fixed-point iterations or other numerical optimization procedures which variational methods typically involve. Another contrast between the variational methods of \cite{Ghahramani1997} and the Graph Filter and Smoother  is that the latter share the recursive-in-time structure of the forward-backward algorithm. The forward pass, which conducts the task of \emph{filtering}, can therefore be used for prediction in online settings. Variational methods for FHMMs work in a batch setting, suitable for offline data analysis.

In taking first steps towards rigorously understanding the quality of factorization-based approximations it is natural and desirable to impose on the FHMM: 
\begin{enumerate}
\item structure which allows one to express connections between the scaling properties of the approximation errors, computational cost and dimensionality, and
\item assumptions that allow the weakness of dependence across dimensions to be quantified.
\end{enumerate}
In a FHMM, by definition, hidden state variables are a-priori independent across dimensions, so any dependence across dimensions under the corresponding posterior distributions must arise from the emission distribution. Therefore the emission distribution is where we impose structure and assumptions as per 1. and 2. above. As made precise in equation \eqref{eq:likelihood_factorization} below, we assume the emission distribution likelihood function has a factorial structure, expressed in terms of a factor graph. Whilst any FHMM can be expressed in this general form (there could be just one factor, which is a function of all dimensions of the hidden state) the Graph Filter and Smoother allow the exponential-in-dimension cost of inference to be significantly reduced when there are several factors, each depending on some ``small'' subset of the hidden state dimensions. Part of the novelty of our theoretical results is that in this setting it turns out that graph distance on the factor graph can be used to quantify the strength of statistical dependence across dimension, and this allows us to express approximation error in terms of quantities  such as the cardinalities of neighborhoods on the factor graph, which do not necessarily grow with overall dimension of the FHMM. The main message conveyed by our theoretical results is then that the Graph Filter and Smoother have  \virg{dimension-free} error performance, in the sense that as the dimension of the underlying state space increases, the assumptions we make do not necessarily become more demanding, and the error bounds we derive do not necessarily degrade. The proofs are heavily influenced by the Dobrushin Comparison techniques used by \cite{rebeschini2015can}.

One situation in which this factorial emission likelihood structure naturally arises is when, at each time step, the observed data consist of a set of spatially-indexed values, and the factor graphs captures this spatial structure, or at least some topological aspects thereof. This is the case, for example, in models of transport flow data \citep{hofleitner2012learning, woodard2017predicting}, where observations measure the numbers of vehicles or passengers at various interconnected locations on a rail, road or flight network, and the hidden variables indicate levels of congestion in spatially localized regions. Here the factor graph arises naturally from the transport network. In this paper we provide numerical results for such a model analyzing data from the London Tube underground rail network. Another domain in which this structure arises is spatio-temporal epidemiology, where the data are noisy, incomplete measurement of disease states across a collection of locations with pairwise dependence \citep{park2020inference, asfaw2021statistical}. Spatial localization of observations is a key element of geophysical models, often treated in terms of continuous hidden variables, with inference performed e.g. using particle filters \citep{van2009particle}. The FHMM structure we consider encompasses discrete-state counterparts of these models. The same is true of high-dimensional multivariate stochastic volatility models with cross-leverage effects \citep{xu2019particle}, where it is common to use discrete states to capture switching between volatility regimes. Discrete-state models also arise when detecting change-points in spatio-temporal data such as temperatures and air-pollution levels \citep{knoblauch2018spatio}. Such models could be cast as FHMMs with factorial emission likelihood structure.


Section \ref{sec:FHMM} introduces the class of FHMMs we consider and background on exact inference by filtering and smoothing. Section \ref{sec:approxFS} introduces the Graph Filter and Smoother, it discusses their computational complexity and our main results, Theorems \ref{th:goodapprox} and \ref{thm:smoothing}. An important component in our analysis is a preliminary result concerning approximate Bayes' rule updates using localized factorial likelihoods, Proposition \ref{prop:goodmarg}, which may be of independent interest. 

Numerical experiments are reported in Section \ref{sec:numerical}. We illustrate that the scaling of errors with algorithm parameters and model attributes indicated by our theoretical results can have a substantial impact in practice by avoiding exponential-in-dimension computational cost. Considering synthetic data so that true parameter values are known, we show how the Graph Smoother can be used within an approximate EM algorithm to fit the parameters of an FHMM and we compare accuracy to the approximate EM approach of \cite{Ghahramani1997} involving variational approximations.

Finally, we illustrate the utility of the Graph Filter and Smoother analysing Oyster card \virg{tap} data on the London Underground train network, estimating parameters and performing prediction of passenger inflow-outflow through stations, showing that Graph Filter-Smoother can be used to model complex spatial interactions over time. Some possible extensions and future research directions are described in Section \ref{sec:conc}. All proofs and a collection of supporting results are in the online appendix (\href{https://github.com/LorenzoRimella/GraphFilter-GraphSmoother/blob/master/Appendix-Exploiting-locality-in-high-dimensional-factorial-hidden-Markov-models.pdf}{\bf Online appendix}). All the codes are implemented in Python 3 and available on Github (\href{https://github.com/LorenzoRimella/GraphFilter-GraphSmoother}{\bf Online code}).

\section{Factorial hidden Markov models, filtering and smoothing} \label{sec:FHMM}

Over a time horizon of length $T \in \mathbb{N}$, a HMM is a pair of processes $(X_t)_{t \in \{0, \dots, T\} }$, $(Y_t)_{t \in \{1, \dots, T\} }$ where the unobserved process $(X_t)_{t \in \{0, \dots, T\} }$ is a Markov chain, and each observation $Y_t$ is conditionally independent of all other variables given $X_t$.

We consider the case where the state-space of $(X_t)_{t \in \{0, \dots, T\} }$ is  of product form, $X_t = (X_t^v)_{v\in V}\in\mathbb{X}^V$ where $\mathbb{X}$ and $V$ are finite sets, and write $L\coloneq\mathbf{card}(\mathbb{X})$ and $M\coloneq\mathbf{card}(V)$. Each $(Y_t)_{t \in \{1, \dots, T\} }$ is valued in a set $\mathbb{Y}$ which could be a discrete set, $\mathbb{R}^d$ or some subset thereof.

We write $\mu_0(x)$ for the probability mass function of $X_0$,  $p(x,z)$ for the transition probability of $(X_t)_{t \in \{0, \dots, T\} }$ from $x$ to $z$, and  $g(x,y)$ for the conditional probability mass or density function of $Y_t$ given $X_t$. In the literature, $g(x,\cdot)$ is often called the emission distribution.

For any $U\subseteq V$ and  $x=(x^v)_{v\in V}\in\mathbb{X}^V $ we shall use the shorthand $x^U=(x^v)_{v\in U}$, and similarly for any probability mass function $\mu$ on $\mathbb{X}^V$ we shall denote its marginal associated with $U$ by $\mu^U$.  When $\mathcal{K}$ is any partition of the set $V$, we shall say that $\mu$  factorizes with respect to $\mathcal{K}$ if:
$$
\mu(x)=\prod_{K\in\mathcal{K}} \mu^K(x^K),\quad \forall x\in \mathbb{X}^V,
$$
and in this situation we will use the shorthand:
$$
\mu = \bigotimes\limits_{K \in \mathcal{K}} \mu^K.
$$

The total variation distance between probability mass functions on $\mathbb{X}^U$, say $\mu$ and $\nu$, is denoted by:
$$
\|\mu-\nu\| \coloneq \sup_{A\in\sigma(\mathbb{X}^U)} |\mu(A)-\nu(A)|,
$$
with the obvious overloading of notation $\mu(A)=\sum_{x^U\in A}\mu(x^U)$ and where $\sigma(\mathbb{X}^U)$ is the power set of $\mathbb{X}^U$.  When $\mu,\nu$  are probability mass functions on $\mathbb{X}^V$ it will be convenient to denote the local total variation (LTV) distance associated with $U\subseteq V$,
$$
\|\mu-\nu\|_U \coloneq \sup_{A\in\sigma(\mathbb{X}^U)} |\mu^U(A)-\nu^U(A)|.
$$
\subsection{Factorial hidden Markov Models} \label{subsec:FHMM}

\paragraph*{}
In FHMMs \citep{Ghahramani1997} the transition probabilities of the hidden Markov chain are assumed to factorize in the following manner:
\begin{equation} \label{eq:prior_factorization}
p(x,z)=\prod_{v\in V}p^v(x^v,z^v),
\end{equation}
where each $p^v(x^v,z^v)$ is a transition probability on $\mathbb{X}$. Directed acyclic graphs showing the conditional independence structures of HMMs and FHMMs are shown in Figure \ref{fig:graph}.
{
\begin{figure}[httb!]
\begin{subfigure}[b]{0.5\textwidth}
\caption{}\label{fig:HMM}
\centering
\resizebox{1\textwidth}{0.125\textheight}{%
\begin{tikzpicture}[-latex, auto, node distance =2 cm and 3cm ,on grid ,state/.style ={ circle ,top color =white , draw , text=blue,font=\bfseries , minimum width =1.3 cm}]
\node[state] (Xt-1) at (0,0)
{$X_{t-1}$};
\node[state] (Yt-1)[below =of Xt-1]
{$Y_{t-1}$};
\node[state] (Xt) [right =of Xt-1] {$X_t$};
\node[state] (Yt)[below =of Xt]
{$Y_{t}$};
\node[state] (Xt+1) [right =of Xt] {$X_{t+1}$};
\node[state] (Yt+1)[below =of Xt+1]
{$Y_{t+1}$};
{};

\path (-2,0) edge node {} (Xt-1);
\path (Xt-1) edge node {} (Xt);
\path (Xt) edge node {} (Xt+1);
\path (Xt+1) edge node {} (8,0);

\path (Xt-1) edge node {} (Yt-1);
\path (Xt) edge node {} (Yt);
\path (Xt+1) edge node {} (Yt+1);
\end{tikzpicture}
}
\end{subfigure}
\begin{subfigure}[b]{0.5\textwidth}
\caption{}\label{fig:FHMM}
\centering
\resizebox{1\textwidth}{0.215\textheight}{%
\begin{tikzpicture}[-latex, auto, node distance =2 cm and 3cm ,on grid ,state/.style ={ circle ,top color =white , draw , text=blue,font=\bfseries , minimum width =1.3 cm}]
\node[state] (Xt-1v) at (0,0)
{$X_{t-1}^v$};
\node[state] (Yt-1)[below =of Xt-1v]
{$Y_{t-1}$};
\node[state] (Xtv) [right =of Xt-1v] {$X_t^v$};
\node[state] (Yt)[below =of Xtv]
{$Y_{t}$};
\node[state] (Xt+1v) [right =of Xtv] {$X_{t+1}^v$};
\node[state] (Yt+1)[below =of Xt+1v]
{$Y_{t+1}$};

\node[state] (Xt-1w)[above =of Xt-1v]
{$X_{t-1}^w$};
\node[state] (Xtw) [above =of Xtv] {$X_t^w$};
\node[state] (Xt+1w) [above =of Xt+1v] {$X_{t+1}^w$};

{};

\path (-2, 0) edge node {} (Xt-1v);
\path (Xt-1v) edge node {} (Xtv);
\path (Xtv) edge node {} (Xt+1v);
\path (Xt+1v) edge node {} (8, 0);

\path (-2,2) edge node {} (Xt-1w);
\path (Xt-1w) edge node {} (Xtw);
\path (Xtw) edge node {} (Xt+1w);
\path (Xt+1w) edge node {} (8, 2);

\path (Xt-1v) edge node {} (Yt-1);
\path (Xtv) edge node {} (Yt);
\path (Xt+1v) edge node {} (Yt+1);

\path (Xt-1w) edge [bend left = 40] node {} (Yt-1);
\path (Xtw) edge [bend left = 40] node {} (Yt);
\path (Xt+1w) edge [bend left = 40] node {} (Yt+1);


\path (-0.5, 3.75) edge [bend right = 40] node {} (Yt-1);
\path (2.5, 3.75) edge [bend right = 40] node {} (Yt);
\path (5.5, 3.75) edge [bend right = 40] node {} (Yt+1);
\end{tikzpicture}
}
\end{subfigure}
\caption{A conditional independence structure of a HMM (\subref{fig:HMM}) and an FHMM (\subref{fig:FHMM})}\label{fig:graph}
\end{figure}
}
As motivated in section \ref{sec:setting_and_cont}, in this paper we consider a factorial structure for the likelihood function $x\mapsto g(x,y)$. Let $F$ be a finite set and let $\mathcal{G}=(V,F,E)$ be a factor graph associated with $x\mapsto g(x,y)$, that is a bi-partite graph with vertex sets $V,F$ and edge set $E$ such that $g(x,y)$ can be written in terms of factors:
\begin{equation} \label{eq:likelihood_factorization}
g(x,y) = \prod\limits_{f \in F} g^f ( x^{N{(f)}},y ),
\end{equation}
where $N(\cdot)$ is the neighbourhood function,
\begin{equation} \label{neighfunc}
N{(w)} \coloneq \{{w'} \in V  \cup F: (w,{w'}) \in E\},\quad w\in V \cup {F}.
\end{equation}
In applications each observation $y$ will typically be multivariate and each likelihood factor $g^f(x^{N{(f)}},y)$ may depend on $y$ only through some subset of its constituent variates, but the details will be model specific, so we do not introduce them at this stage. We note that equations \ref{eq:prior_factorization} and \ref{eq:likelihood_factorization} are the mathematical formalization of \virg{local structures}, indeed equation \ref{eq:prior_factorization} says that each component of the hidden variable is evolving independently, while equation \ref{eq:likelihood_factorization} stands that the emission distribution can be decomposed in pieces that can be built locally.

\begin{figure}[httb!]
\centering
\resizebox{0.55\textwidth}{0.165\textheight}{%
\begin{tikzpicture}[on grid ,state/.style ={ circle ,top color =white , draw , text=blue,font=\bfseries , minimum width =1 cm}]



\draw[thick] (-1,0) -- (-1,0) node[left,blue] {$X_t$};
\draw[thick] (-1,-3) -- (-1,-3) node[left,blue] {$Y_t$};


\node[state] (Xt1) at (0,0) {$X_t^{1}$};
\node[state] (Xt2) at (2,0) {$X_t^{2}$};
\node[state] (Xt3) at (4,0) {$X_t^{3}$};
\node[state] (Xt4) at (6,0) {$X_t^{4}$};
\node[state] (Xt5) at (8,0) {$X_t^{5}$};

\node[state, diamond, text=red] (f1) at (0,-2-1) {$f_1$};
\node[state, diamond, text=red] (f2) at (2.75,-2-1) {$f_2$};
\node[state, diamond, text=red] (f3) at (5.25,-2-1) {$f_3$};
\node[state, diamond, text=red] (f4) at (8,-2-1) {$f_4$};

\path (Xt1) edge [red] node {} (f1);
\path (Xt2) edge [red] node {} (f1);
\path (Xt1) edge [red] node {} (f2);
\path (Xt2) edge [red] node {} (f2);
\path (Xt3) edge [red] node {} (f2);
\path (Xt3) edge [red] node {} (f3);
\path (Xt4) edge [red] node {} (f3);
\path (Xt4) edge [red] node {} (f4);
\path (Xt5) edge [red] node {} (f4);

\draw[blue,thick,dotted] (0-0.75,0+0.75)  rectangle (8+0.75,0-0.75);
\draw[blue,thick,dotted] (0-0.75,-2+0.75-1)  rectangle (8+0.75,-2-0.75-1);

\end{tikzpicture}
}
\caption{An example of factor graph for a fixed time step $t$ of an FHMM where $V= \{ 1,2,3,4,5 \}$ and $F= \{ f_1,f_2,f_3,f_4 \}$.} \label{fig:factorgraph}
\end{figure}

\subsection{Filtering and Smoothing}\label{subsec:exact_filt_smooth}
The tasks of {filtering} and {smoothing} at time $t$ are to compute, respectively, the conditional distributions of $X_t$ given the realized observations $(y_1,\ldots,y_t)$ and $(y_1,\ldots,y_T)$. We shall denote the corresponding probability mass functions by $\pi_t$ and $\pi_{t|T}$. Aswell as facilitating inference about hidden states, computing these distributions is a often a key step in estimating static parameters of the HMM.

Filtering can be conducted in a forward pass through the data:
\begin{equation} \label{filtering}
\pi_0 \coloneq \mu_0, \qquad
\pi_t \coloneq \mathsf{F}_t \pi_{t-1},\qquad \mathsf{F}_t \coloneq \mathsf{C}_t \mathsf{P} , \qquad t \in \{1,\dots,T\},
\end{equation}
where the \virg{prediction} operator $\mathsf{P}$ and the \virg{correction} operator $\mathsf{C}_t$ act on probability mass functions $\mu$ as:
\begin{equation} \label{eq:KandC}
\mathsf{P} \mu (x) \coloneq \sum\limits_{z \in \mathbb{X}^V} p(z,x) \mu({z}),
\qquad
\mathsf{C}_t \mu (x) \coloneq \frac{g(x,y_t) \mu({x})}{\sum\limits_{z \in \mathbb{X}^V} g(z,y_t) \mu({z})},\qquad x \in \mathbb{X}^V.
\end{equation}
If $\mu$ is considered a prior distribution, the operator $\mathsf{C}_t$ can be understood as applying a Bayes' rule update using the likelihood function $g(x,y_t)$.

Amongst various smoothing algorithms, we focus on the forward-filtering, backward-smoothing method presented by \cite{kitagawa1987non}, which involves a backward in time recursion performed after filtering:
\begin{equation} \label{smoothing}
\pi_{T|T} \coloneq \pi_T ,\qquad
\pi_{t|T} \coloneq \mathsf{R}_{\pi_t} \pi_{t+1|T}, \qquad t \in \{T-1,\ldots,0\},
\end{equation}
where for probability mass functions $\mu,\nu$ the operator $\mathsf{R}_{\nu}\mu$ is defined as:
\begin{equation*} \label{operatorR}
\mathsf{R}_{\nu}\mu (x) \coloneq  \sum\limits_{z \in \mathbb{X}^V} \frac{p(x,z) \nu({x})}{\sum_{\tilde{x} \in \mathbb{X}^V} p(\tilde{x},z) \nu(\tilde{x})} \mu(z),\qquad x \in \mathbb{X}^V.
\end{equation*}

\paragraph*{}
\cite{Ghahramani1997} showed that for FHMMs the complexity of  \eqref{filtering} and  \eqref{smoothing} together is  $\mathcal{O}(T M L^{M+1})$. The $L^{M}$ part of this complexity makes implementation prohibitively costly as the dimension $M$ grows.

\section{Approximate filtering and smoothing}
\label{sec:approxFS}

To introduce our approximate filtering and smoothing techniques --- called the Graph Filter and Graph Smoother --- consider the filtering and smoothing recursions, \eqref{filtering} and \eqref{smoothing}, and fix any partition $\mathcal{K}$ of $V$. Suppose that one has already obtained an approximation to $\pi_{t-1}$, call it $\tilde{\pi}_{t-1}$,   which factorizes with respect to $\mathcal{K}$. Then due to \eqref{eq:prior_factorization}, $\mathsf{P} \tilde{\pi}_{t-1}$, also factorizes with respect to $\mathcal{K}$. However $\mathsf{C}_t \mathsf{P} \tilde{\pi}_{t-1}$ does not factorize with respect to $\mathcal{K}$ in general. In Section \ref{subsec:approx_bayes} we shall define an approximation to the Bayes update operator $\mathsf{C}_t$,  denoted $\tilde{\mathsf{C}}_t^m$, where $m$ is a parameter, such that  $\tilde{\pi}_{t}\coloneq \tilde{\mathsf{C}}_t^m \mathsf{P} \tilde{\pi}_{t-1}$ does factorize with respect $\mathcal{K}$.

Once $(\tilde{\pi}_{t})_{t\in\{0,\ldots,T\}}$ have been computed in this manner, a sequence of approximate smoothing distributions $(\tilde{\pi}_{t
|T})_{t\in\{0,\ldots,T\}}$ will be obtained by setting $\tilde{\pi}_{T|T}\coloneq\tilde{\pi}_{T}$ and then $\tilde{\pi}_{t-1|T}\coloneq \mathsf{R}_{\tilde{\pi}_t}\tilde{\pi}_{t|T}$, for $t=T,T-1,\dots,0$.  Due to \eqref{eq:prior_factorization} and the fact that $(\tilde{\pi}_t)_{t\in\{0,\ldots,T\}}$ each factorizes with respect to $\mathcal{K}$,  it follows that  $(\tilde{\pi}_{t|T})_{t\in\{0,\ldots,T\}}$ also factorize with respect to $\mathcal{K}$.

The key ingredient in all of this is finding a way to approximate the action of the Bayes update operator $\mathsf{C}_t$ in an accurate but computationally inexpensive manner. Our next objective is to introduce the details of how we do so.

\subsection{Approximate Bayes updates via localization and factorization}\label{subsec:approx_bayes}
Let $d:(V\cup F)^2\to \mathbb{N}$ be the graph distance on $\mathcal{G}$, that is $d(w,w')$ is the number of edges in a shortest path between $w,w'$. Augmenting the definition of the neighborhood function \eqref{neighfunc}, define, for any $J \subseteq V$,
\begin{align*}
N_{v}^r(J) &\coloneq \{ v' \in V \text{ such that } \exists v \in J \text{ with }  d(v,v') \leq 2r+2\},\\
N_f^r(J) &\coloneq \{ f \in F \text{ such that } \exists v \in J \text{ with }  d(v,f) \leq 2r+1\}.
\end{align*}
The sets $N_{v}^r(J)$ and $N_f^r(J)$ will play an important role in our theoretical results, we need these two definitions due tot he bi-partite nature of the graph $\mathcal{G}$. 

For a given probability mass function $\mu$ on $\mathbb{X}^V$, a partition of $V$ denoted $\mathcal{K}$ and $m\geq0$ define:
\begin{align}
\tilde{\mathsf{C}}^{m,K}_t \mu(x^K) &\coloneq
\frac{\sum_{{z} \in \mathbb{X}^{V}: z^K=x^K} \prod_{f \in {N_f^m(K)}} g^f {(} z^{N{(f)}}, y_t{)} \mu(z)}{\sum_{z \in \mathbb{X}^V} \prod_{f \in {N_f^m(K)}} g^f {(}z^{N{(f)}}, y_t {)} \mu(z)}, \quad x^K \in \mathbb{X}^K, K \in \mathcal{K},\label{eq:localizfact}\\
\tilde{\mathsf{C}}^{m}_t \mu &\coloneq \bigotimes\limits_{K \in \mathcal{K}}\tilde{\mathsf{C}}^{m,K}_t \mu.\label{eq:productmarg}
\end{align}
Note the dependence of $\tilde{\mathsf{C}}^{m}_t$ on $\mathcal{K}$ is not shown in the notation. Here $m$ is key parameter which, as we shall see in more detail later, influences both the accuracy and computational cost of our approximate inference algorithms. 

To see the motivation for \eqref{eq:localizfact}-\eqref{eq:productmarg}, observe from the definition of the Bayes update operator \eqref{eq:KandC} and factorial likelihood function \eqref{eq:likelihood_factorization} that the marginal distribution of $\mathsf{C}_t\mu$ associated with some $K \in \mathcal{K}$ is given by:
\begin{equation}\label{eq:bayes_marginal}
(\mathsf{C}_t \mu)^K (x^K) \coloneq
\frac{\sum_{{z} \in \mathbb{X}^{V}: z^K=x^K} \prod_{f \in F} g^f {(} z^{N{(f)}}, y_t{)} \mu(z)}{\sum_{z \in \mathbb{X}^V} \prod_{f \in F} g^f {(}z^{N{(f)}}, y_t {)} \mu(z)}.
\end{equation}
The definition \eqref{eq:localizfact}-\eqref{eq:productmarg} thus embodies two ideas: \emph{localization}, in that $\tilde{\mathsf{C}}^{m,K}_t \mu$ is an approximation to the exact marginal  $(\mathsf{C}_t \mu)^K $ obtained by replacing the likelihood function $\prod_{f \in F} g^f {(} z^{N{(f)}}, y_t{)}$ in \eqref{eq:bayes_marginal} by the \virg{local-to-$K$} product $\prod_{f \in {N_f^m(K)}} g^f {(} z^{N{(f)}}, y_t{)}$; and \emph{factorization}, in that $\tilde{\mathsf{C}}^{m}_t \mu$  factorizes with respect to $\mathcal{K}$ by construction. Figure \ref{fig:factorgraphprob1} illustrates sub-graphs of $\mathcal{G}$ associated with each of the neighborhoods $N_f^m(K)$, $K\in\mathcal{K}$.

It is important to note that the factorization idea alone is not enough for our purposes: computing the marginal distribution $(\mathsf{C}_t \mu)^K$ has a cost which is exponential in $M$ in general for likelihoods of the form \eqref{eq:likelihood_factorization}, even when $\mu$ factorizes with respect to $\mathcal{K}$. So taking $\bigotimes_{K \in \mathcal{K}}(\mathsf{C}_t \mu)^{K}$ as an approximation to $\mathsf{C}_t \mu$ would offer no computational advantage. This distinguishes our setup from that of \citep{rebeschini2015can,finke2017approximate} as discussed in Section \ref{sec:intro}, and is the reason we introduce localization through the parameter $m$.

\begin{figure}[httb!]
\centering
\resizebox{0.75\textwidth}{0.40\textheight}{%
\begin{tikzpicture}[on grid ,state/.style ={ circle ,top color =white , draw , text=blue,font=\bfseries , minimum width =1 cm}]


\node[state] (Xt1) at (0,0) {$X^{1}$};
\node[state] (Xt2) at (2,0) {$X^{2}$};
\node[state] (Xt3) at (4,0) {$X^{3}$};
\node[state] (Xt4) at (6,0) {$X^{4}$};
\node[state] (Xt5) at (8,0) {$X^{5}$};

\node[state] (sXt1) at (-3,-4.6) {$X^{1}$};
\node[state] (sXt2) at (-1,-4.6) {$X^{2}$};
\node[state] (sXt3) at (1,-4.6) {$X^{3}$};
\node[state, diamond, text=red] (sf1) at (-3,-4.6-1.5) {$f_1$};
\node[state, diamond, text=red] (sf2) at (1,-4.6-1.5) {$f_2$};

\path (sXt1) edge [red] node {} (sf1);
\path (sXt2) edge [red] node {} (sf1);
\path (sXt1) edge [red] node {} (sf2);
\path (sXt2) edge [red] node {} (sf2);
\path (sXt3) edge [red] node {} (sf2);

\node[state] (ssXt1) at (-3,2.25+2) {$X^{1}$};
\node[state] (ssXt2) at (-1,2.25+2) {$X^{2}$};
\node[state] (ssXt3) at (1,2.25+2) {$X^{3}$};
\node[state] (ssXt4) at (3,2.25+2) {$X^{4}$};
\node[state, diamond, text=red] (ssf1) at (-3,0.5+2.25) {$f_1$};
\node[state, diamond, text=red] (ssf2) at (0,0.5+2.25) {$f_2$};
\node[state, diamond, text=red] (ssf3) at (3,0.5+2.25) {$f_3$};

\path (ssXt1) edge [red] node {} (ssf1);
\path (ssXt2) edge [red] node {} (ssf1);
\path (ssXt1) edge [red] node {} (ssf2);
\path (ssXt2) edge [red] node {} (ssf2);
\path (ssXt3) edge [red] node {} (ssf2);
\path (ssXt3) edge [red] node {} (ssf3);
\path (ssXt4) edge [red] node {} (ssf3);

\node[state] (sssXt3) at (8,-4.6) {$X^{3}$};
\node[state] (sssXt4) at (10,-4.6) {$X^{4}$};
\node[state] (sssXt5) at (12,-4.6) {$X^{5}$};
\node[state, diamond, text=red] (sssf2) at (8,-4.75-1.35) {$f_2$};
\node[state, diamond, text=red] (sssf3) at (12,-4.75-1.35) {$f_3$};
\path (sssXt3) edge [red] node {} (sssf2);
\path (sssXt4) edge [red] node {} (sssf3);
\path (sssXt4) edge [red] node {} (sssf2);
\path (sssXt5) edge [red] node {} (sssf3);

\node[state] (ssssXt4) at (9,4.25) {$X^{4}$};
\node[state] (ssssXt5) at (11,4.25) {$X^{5}$};
\node[state, diamond, text=red] (ssssf3) at (10,4.25-1.5) {$f_4$};
\path (ssssXt4) edge [red] node {} (ssssf3);
\path (ssssXt5) edge [red] node {} (ssssf3);

\node[state, diamond, text=red] (f1) at (0,-2) {$f_1$};
\node[state, diamond, text=red] (f2) at (2.75,-2) {$f_2$};
\node[state, diamond, text=red] (f3) at (5.25,-2) {$f_3$};
\node[state, diamond, text=red] (f4) at (8,-2) {$f_4$};

\path (Xt1) edge [red] node {} (f1);
\path (Xt2) edge [red] node {} (f1);
\path (Xt1) edge [red] node {} (f2);
\path (Xt2) edge [red] node {} (f2);
\path (Xt3) edge [red] node {} (f2);
\path (Xt3) edge [red] node {} (f3);
\path (Xt4) edge [red] node {} (f3);
\path (Xt4) edge [red] node {} (f4);
\path (Xt5) edge [red] node {} (f4);

\draw[blue,thick,dotted] (0-0.75,0+0.75)  rectangle (0+2.75,0-0.75);
\draw[blue,thick,dotted] (4-0.75,0+0.75)  rectangle (4+0.75,0-0.75);
\draw[blue,thick,dotted] (6-0.75,0+0.75)  rectangle (6+0.75,0-0.75);
\draw[blue,thick,dotted] (8-0.75,0+0.75)  rectangle (8+0.75,0-0.75);

\draw[green,thick,dotted] (0-1,0+1)  rectangle (0+9,0-2.85);

\draw[green,thick,dotted] (-3-0.75,0-4+0.15)  rectangle (+1+0.75,0-7+0.15);
\draw[green,thick,dotted] (-3-0.75,0+2)  rectangle (3+0.75,0+5);
\draw[green,thick,dotted] (8-0.75,0-4+0.15)  rectangle (12+0.75,0-7+0.15);
\draw[green,thick,dotted] (+9-0.75,0+2)  rectangle (+11+0.75,0+5);


\path (0-0.75,0.05) edge [blue,thick,dotted,-latex, bend right=30] node {} (-2,0-3.9);
\path (3.9,1-0.2) edge [blue,thick,dotted,-latex, bend right=0] node {} (2,2);
\path (6,-0.75) edge [blue,thick,dotted,-latex, bend right=0] node {} (8,0-3.8);
\path (9.5-0.75,0) edge [blue,thick,dotted,-latex, bend right=30] node {} (10.15,2);

\end{tikzpicture}
}
\caption{An example of sub-graphs associated with the neighbourhoods $N_f^m(K)$, $K\in\mathcal{K}$, when $V= \{1,2,3,4,5\}$, $F=\{f_1,f_2,f_3,f_4\}$, $\mathcal{K}=\{\{1,2\},\{3\},\{4\},\{5\}\}$ and $m=0$.} \label{fig:factorgraphprob1}
\end{figure}

More detailed consideration of the complexity of computing $\tilde{\mathsf{C}}^{m}_t \mu$ is given after our first main result, Proposition \ref{prop:goodmarg}, which quantifies the approximation error associated with $\tilde{\mathsf{C}}^{m}_t$ and is one of the building blocks in the overall analysis of our approximate filtering and smoothing method.

In order  to state Proposition \ref{prop:goodmarg} we need some further definitions.  Firstly let us introduce the following attributes of the factor graph $\mathcal{G}$.
\begin{align}
d(J,J') &\coloneq \min\limits_{w \in J} \min\limits_{w' \in J'} d(w,w'),\quad J,J' \subseteq V \cup F,\label{eq:graph_quants1}\\
n_K &\coloneq \frac{1}{2} \max_{ v \in V} d(K,v) , \quad K\in\mathcal{K},\\
n &\coloneq \max_{K \in \mathcal{K}} n_K ,\\
\Upsilon  & \coloneq \max\limits_{v \in V}  \mathbf{card}(N(v)),\\
\Upsilon^{(2)}&  \coloneq \max\limits_{v \in V}  \mathbf{card}(N^0_v(v)), \\
\tilde{\Upsilon}& \coloneq\max\limits_{v,v' \in V}  \mathbf{card}(N(v) \cap N(v')). \label{eq:graph_quants2}
\end{align}
Note the dependence of $n$ on $\mathcal{K}$ is not shown in the notation.

Secondly,  given a probability mass function $\mu$ on $\mathbb{X}^V$ and a random variable $X \sim \mu$, we shall denote by $\mu^v_x$ the conditional distribution of $X^v$ given $\{X^{V \setminus v}=x^{V \setminus v}\}$,   and define
\begin{align*}
C^{\mu}_{v,v'}&\coloneq \frac{1}{2} \sup_{x,z \in \mathbb{X}^V:x^{V \setminus v'} = z^{V \setminus v'}} \norm{ \mu^v_x- \mu^v_z }, \quad v,v' \in V,\\
{\mbox{Corr}}(\mu,\beta) &\coloneq \max_{v \in V} \sum_{v' \in V} e^{\beta d(v,v')}{C}^{\mu}_{v,v'},
\end{align*}
where $\beta>0$ is a given constant.

\begin{proposition}\label{prop:goodmarg}
Fix any partition $\mathcal{K}$ of $V$ and any $t\in\{1,\ldots,T\}$. Suppose there exists $\kappa \in (0,1)$ such that:
\begin{equation} \label{eq:Prop1hyp_g}
\kappa \leq g^f \left ( {x}^{N(f)},y_t \right ) \leq \frac{1}{\kappa}, \quad \forall x \in \mathbb{X}^V, f\in F.
\end{equation}
Assume that for a given probability mass function $\mu$ on $\mathbb{X}^V$ there exists $\beta>0$ such that:
\begin{equation} \label{eq:Prop1hyp_Corr}
2 \kappa^{-2 {\Upsilon}} {\mbox{Corr}}(\mu,\beta)+ e^{2 \beta} \Upsilon^{(2)} \left ( 1- \kappa^{2 \tilde{\Upsilon}} \right ) \leq \frac{1}{2}.
\end{equation}
Then for any $K \in \mathcal{K}$, $J \subseteq K$ and $m \in \{0,\dots,n\}$,
\begin{equation}
\norm{\mathsf{C}_t \mu-\tilde{\mathsf{C}}^m_t \mu}_J \leq
4 e^{-\beta} \left ( 1-\kappa^{b(m, \mathcal{K})}\right ) \mathbf{card}(J) e^{-\beta
 m},\label{eq:Prop1_bound}
\end{equation}
where $ b(m, \mathcal{K}) \coloneq 2 \max_{K \in \mathcal{K}}\max_{v \notin { N^{m-1}_v(K)}} \{\mathbf{card}(N(v))\}$, with the convention that the maximum over an empty set is zero.
\end{proposition}
The proof of Proposition \ref{prop:goodmarg} is given in appendix A. The term ${\mbox{Corr}}(\mu,\beta)$ quantifies the strength of dependence across the coordinates of $X=(X^v)_{v\in{V}}\sim\mu$ . We say that $\mu$ satisfies the decay of correlation if it exists a $\beta$ such that ${\mbox{Corr}}(\mu,\beta)$ is bounded above. The hypothesis \eqref{eq:Prop1hyp_Corr} places a combined constraint on this dependence, the constant $\kappa$ which in \eqref{eq:Prop1hyp_g} controls the oscillation of the likelihood function factors $g^f({x}^{N(f)},y_t)$, and the graph attributes $\Upsilon$,  $\Upsilon^{(2)}$ and $\tilde{\Upsilon}$.

Turning to the bound \eqref{eq:Prop1_bound}, let us examine its dependence on the partition $\mathcal{K}$ and the parameter $m$. The quantity $b(m,\mathcal{K})$ is non-increasing with $m$. In practice,   $b(m,\mathcal{K})$  will often be decreasing with $m$,  and is always zero when $m=n$ since then $N_v^{m-1}(K)$ is $V$. Also $b(m,\mathcal{K})$ will often decrease as $\mathcal{K}$ becomes more coarse and in the extreme case of the trivial partition $\mathcal{K}=\{V\}$, the constant $b(m,\mathcal{K})$ is always zero, because $N_v^{m-1}(K)$ is again $V$, and so $v \notin V$ is equivalent to $v \in \emptyset$.  Combined with the $e^{-\beta m}$ term, this means $\norm{\mathsf{C}_t \mu-\tilde{\mathsf{C}}^m_t \mu}_J $ can be made small by choosing the  partition $\mathcal{K}$ to be suitably coarse and $m$ to be suitably large.

It is important to note that ${\mbox{Corr}}(\mu,\beta)$, $\kappa$ and $\Upsilon$ appearing in the hypotheses \eqref{eq:Prop1hyp_g} and \eqref{eq:Prop1hyp_Corr}, and  the quantities on the right hand side of \eqref{eq:Prop1_bound} do not necessarily have any dependence on $M$, the overall dimension of the state-space. For instance, when $\mu=\otimes_{v\in V} \mu^v$ then ${\mbox{Corr}}(\mu,\beta)=0$ and one can easily construct families of FHMMs of increasing dimension in which $\kappa$, $\Upsilon$, $\Upsilon^{(2)}$, $\tilde{\Upsilon}$, $b(m,\mathcal{K})$ are independent of $M$:  consider the simple case where the factor graph is a chain as shown in Figure \ref{fig:fgsimple}, and the dimension of the model is increased by adding $f_5$ and $X^{(6)}$ then $f_6$ and $X^{(7)}$ as shown by the dashed lines.  In this situation, for any $v \in V$ the cardinality of $N(v)$ and $N^0_v(v)$ remain unchanged as the dimension of the model increases.

\begin{figure}[httb!]
\centering
\resizebox{0.8\textwidth}{0.125\textheight}{%
\begin{tikzpicture}[on grid ,state/.style ={ circle ,top color =white , draw , text=blue,font=\bfseries , minimum width =1 cm}]

\node[state, densely dotted] (Xt7) at (5.25,0) {$X_t^{7}$};
\node[state, densely dotted] (Xt6) at (2.75,0) {$X_t^{6}$};
\node[state] (Xt1) at (0,0) {$X_t^{5}$};
\node[state] (Xt2) at (-2,0) {$X_t^{4}$};
\node[state] (Xt3) at (-4,0) {$X_t^{3}$};
\node[state] (Xt4) at (-6,0) {$X_t^{2}$};
\node[state] (Xt5) at (-8,0) {$X_t^{1}$};

\node[state, diamond, text=red, densely dotted] (f6) at (5.25,2) {$f_6$};
\node[state, diamond, text=red, densely dotted] (f5) at (2.75,2) {$f_5$};
\node[state, diamond, text=red] (f1) at (0,2) {$f_4$};
\node[state, diamond, text=red] (f2) at (-2.75,2) {$f_3$};
\node[state, diamond, text=red] (f3) at (-5.25,2) {$f_2$};
\node[state, diamond, text=red] (f4) at (-8,2) {$f_1$};

\path (Xt1) edge [red] node {} (f1);
\path (Xt2) edge [red] node {} (f1);
\path (Xt2) edge [red] node {} (f2);
\path (Xt3) edge [red] node {} (f2);
\path (Xt3) edge [red] node {} (f3);
\path (Xt4) edge [red] node {} (f3);
\path (Xt4) edge [red] node {} (f4);
\path (Xt5) edge [red] node {} (f4);
\path (Xt1) edge [red, densely dotted] node {} (f5);
\path (Xt6) edge [red, densely dotted] node {} (f5);
\path (Xt6) edge [red, densely dotted] node {} (f6);
\path (Xt7) edge [red, densely dotted] node {} (f6);

\end{tikzpicture}
}
\caption{Solid lines indicate a chain factor graph with 4 likelihood factors and $V=\{1,2,3,4,5\}$, hence $M=5$. Dashed lines indicate extension to $M=6,7$ by adding $f_5$ and $X^{(6)}$ then $f_6$ and $X^{(7)}$.} \label{fig:fgsimple}
\end{figure}

\begin{algorithm}[h]
	\caption{Approximate Bayes update} \label{alg:margcon}
	\begin{algorithmic}[1]
		\Require $\mathcal{K},(N^m_f(K))_{K \in \mathcal{K}}, (N^m_v(K))_{K \in \mathcal{K}}, (\mu^K)_{K \in \mathcal{K}}, (g^f(\cdot, y))_{f \in F}$
		\For{$K \in \mathcal{K}$}
			\State{$\hat{K} \leftarrow \{K' \in \mathcal{K}: K' \cap N^m_v(K) \neq \emptyset\}$}
			\For{$x \in \mathbb{X}^{\hat{K}}$}
				\State{ $\hat{\mu}(x) \leftarrow \prod_{K' \in \hat{K}} \mu^{K'}(x^{K'})$}
				\For{$f \in N^m_f(K)$}
					\State{${\hat{\mu}}(x)  \leftarrow \hat{\mu}(x) \cdot g^f \left (x^{N{(f)}}, y \right )$}
				\EndFor
			\EndFor			
			\State{\bf{Normalize ${\hat{\mu}}$ to a probability mass function on} $\mathbb{X}^{\hat{K}}$ }
			\State{\bf{Marginalize out components $\hat{K}\setminus K$:} ${\tilde{\mu}^K} \leftarrow {\hat{\mu}^K}$}
		\EndFor
		\Return{$(\tilde{\mu}^{K})_{K \in \mathcal{K}}$}
	\end{algorithmic}
\end{algorithm}

Algorithm \ref{alg:margcon} shows the steps involved in computing $\tilde{\mathsf{C}}_t^m \mu$ in the case that $\mu$ factorizes with respect to $\mathcal{K}$. To simplify considerations of the computational cost of Algorithm \ref{alg:margcon}, let us suppose that for each  $K \in \mathcal{K}$  there exists a collection of elements in $\mathcal{K}$ that is a partition of $N^m_v(K)$. This is a typical feature of regular graphs such as lattices. In this case the complexity of  Algorithm \ref{alg:margcon} is readily found to be:
$$
\mathcal{O}\left( \mathbf{card}(\mathcal{K}) \max\limits_{K \in \mathcal{K}}\mathbf{card}\left(N^m_f(K)\right) L^{ \max\limits_{K \in \mathcal{K}}\mathbf{card}(N^m_v(K)) }\right).
$$
Crucially the exponent of $L$ in this cost, which is proportional to $\max_{K \in \mathcal{K}}\mathbf{card}(N^m_v(K))$, does not necessarily grow with the overall dimension $M$; recall that $N^m_v(K)$ is, in words, the set of vertices belonging to $V$ which are within $2m+2$ graph distance of the set $K$. Thus $\mathbf{card}(N^m_v(K))$ captures the density of edges in the graph in the neighborhood (defined by $m$) of $K$. Clearly this is a local rather than global characteristic of the graph. Taking the $\max_{K \in \mathcal{K}}$ of these cardinalities is a simple upper-bound on these local quantities across the graph. In the case of regular graphs such as lattices and in particular the chain example in Figure \ref{fig:fgsimple}, the term $\max_{K \in \mathcal{K}}\mathbf{card}(N^m_v(K))$ is independent of the total number of vertices in the graph.   These complexity considerations for  Algorithm \ref{alg:margcon} suggest that the overall cost of a filtering and smoothing method built around the approximate Bayes update operator $\tilde{\mathsf{C}}_t^m$ may avoid the exponential-in-$M$ factor in the cost $\mathcal{O}(T M L^{M+1})$ of exact filtering and smoothing for FHMMs.


\subsection{Graph Filter}
As an approximation to the operator $\mathsf{F}_t$ introduced in Section \ref{subsec:exact_filt_smooth} we now define:
$$
\tilde{\mathsf{F}}^m_t \coloneq \tilde{\mathsf{C}}^m_t \mathsf{P},
$$
where the dependence of $\tilde{\mathsf{F}}^m_t$ on $\mathcal{K}$, inherited from $\tilde{\mathsf{C}}^m_t$, is not shown in the notation. The approximate filtering distributions are then defined by the recursion:
\begin{equation} \label{newfiltering}
\tilde{\pi}_{0}\coloneq \mu_0,\qquad
\tilde{\pi}_{t}\coloneq \tilde{\mathsf{F}}^m_t \tilde{\pi}_{t-1}, \quad t \in \{1,\dots,T\}
\end{equation}

Our next result, Theorem \ref{th:goodapprox},  builds from Proposition \ref{prop:goodmarg} and quantifies the approximation error associated with $(\tilde{\pi}_t)_{t\in\{0,\ldots,T\}}$. In order to state it we need to  introduce, further to \eqref{eq:graph_quants1}-\eqref{eq:graph_quants2}, the definitions for $J \subseteq V$,
\begin{align}
\widetilde{J} &\coloneq  \{v \in J: \forall f \in N(v),\, N(f) \subseteq J \},\label{eq:J_tilde_notation}\\
\partial J &\coloneq J \setminus \widetilde{J},\\
\partial N(J) &\coloneq \{f \in N(J): N(f) \cap V \setminus J \neq \emptyset\}.
\end{align}

Moreover, given a probability mass function $\mu$ on $\mathbb{X}^V$, transition probabilities $p$ on $\mathbb{X}^V$ and two random variables such that $X \sim \mu$ and $Z|X=x \sim p(x,\bullet)$ for $x \in \mathbb{X}^V$, we shall denote by $\mu_{x,{z}}^{v}$ the conditional distribution of $X^v$ given $\{X^{V \setminus v} = X^{V \setminus v}, Z=z\}$, and define
\begin{align}
& \tilde{C}_{v,v'}^{\mu} \coloneqq \frac{1}{2} \sup\limits_{{z} \in \mathbb{X}^V}  {\sup\limits_{\substack{x,\hat{x} \in \mathbb{X}^{V}:\\ x^{V \setminus v'}=\hat{x}^{V \setminus v'}}} \norm{\mu_{x,{z}}^{v}-\mu_{\hat{x},{z}}^{v}}}, \quad v,v' \in V	\\
& \widetilde{\mbox{Corr}}(\mu, \beta) \coloneqq \max_{v \in V} \sum\limits_{v' \in V} e^{\beta d(v,v')}C_{v,v'}^{\mu},
\end{align}
where $\beta>0$ is a given constant. We note that if the components of $X$ are independent, i.e. $\mu = \bigotimes_{v \in V} \mu^v$, then $\tilde{C}_{v,v'}^{\mu} =0$ for any $v \neq v'$ and so $\widetilde{\mbox{Corr}}(\mu, \beta)=0$ for any $\beta>0$.

\begin{theorem} \label{th:goodapprox}
Fix any collection of observations $\{y_1,\dots,y_T\}$ and any partition $\mathcal{K}$ of $V$. There exists a region $\mathcal{R}_0 \subseteq (0,1)^3$  depending only on $\tilde{\Upsilon}$,$\Upsilon$ and $\Upsilon^{(2)}$, such that if, for given $(\epsilon_{-},\epsilon_{+},\kappa) \in \mathcal{R}_0$,
$$
\epsilon_{-} \leq p^v(x^v,z^v)  \leq {\epsilon_{+}},
\quad
\text{and}
\quad
\kappa \leq g^f\left ({x}^{N{(f)}},y_t \right ) \leq \frac{1}{\kappa},
$$
for all $x, z \in\mathbb{X}^V, v \in V, f \in F, t \in \{1,\dots,T\}$, then for $\beta >0$ small enough depending only on $\tilde{\Upsilon},\Upsilon,\Upsilon^{(2)}, \epsilon_{-},\epsilon_{+}$ and $\kappa$, we have that for any $\mu_0$ satisfying: 
\begin{equation}
\widetilde{\mbox{Corr}}(\mu_0,\beta) \leq 2 e^{-\beta} \left ( 1 - \frac{\epsilon_{-}}{\epsilon_{+}} \right ) + 2 e^{2 \beta} \Upsilon^{(2)} \left ( 1 - \kappa^{2 \tilde{\Upsilon}} \right )
\end{equation}
and for any $K \in \mathcal{K}$, $J \subseteq K$ and $m \in \{0,\dots,n\}$:
\begin{equation*}
\begin{split}
\norm{\pi_t- \tilde{\pi}_t}_J
&\leq
\alpha_1(\beta)\left ( 1-\kappa^{a(\mathcal{K})} \right )\mathbf{card}(J)+ \gamma_1(\beta) \left ( 1-\kappa^{b(\mathcal{K},m)} \right ) \mathbf{card}(J) e^{-\beta m}, \quad \forall t \in \{1,\dots,T\},
\end{split}
\end{equation*}
where $\pi_t, \tilde{\pi}_t$ are given by \eqref{filtering} and \eqref{newfiltering} with initial state distribution $\mu_0$; $\alpha_1(\beta), \gamma_1(\beta)$ are constants depending only on $\beta$, and
\begin{align*}
a(\mathcal{K}) &\coloneq 2 \max_{K \in \mathcal{K}} \max_{v \in \partial K} \mathbf{card}(N(v) \cap \partial N(K)),\\
b(m,\mathcal{K}) & \coloneq 2 \max_{K \in \mathcal{K}}\max_{v \notin { N^{m-1}_v(K)}} \mathbf{card}(N(v)),
\end{align*}
with the convention that the maximum over an empty set is zero.
\end{theorem}

The proof of Theorem \ref{th:goodapprox} is in appendix A. Explicit expressions for $\mathcal{R}_0$, $\beta$, $\alpha_1(\beta)$ and $ \gamma_1(\beta)$ are given in the proof of the theorem and its supporting results. The full assumption on the initial distribution $\mu_0$ can be found in appendix A.

The second term on the right hand side of the bound on $\norm{\pi_t- \tilde{\pi}_t}_J$ given in Theorem \ref{th:goodapprox} is, up to a numerical constant, equal to the upper bound obtained in Proposition \ref{prop:goodmarg}, see discussion there of its dependence on $m$ and $\mathcal{K}$. The first term on the right hand side of the bound on $\norm{\pi_t- \tilde{\pi}_t}_J$ depends on the neighborhood structure of the factor graph $\mathcal{G}$. Loosely speaking, the constant $a(\mathcal{K})$ is small when the graph is sparsely connected and the partition is coarse, and in the extreme case of the trivial partition $\mathcal{K}=\{V\}$, the constant $a(\mathcal{K})$ is zero because, in the notation of \eqref{eq:J_tilde_notation}, $\tilde{V}=V$ and so $\partial V$ is empty. The quantities in the hypotheses and bound of Theorem \ref{th:goodapprox} exhibit the same dimension-free qualities as discussed after Proposition \ref{prop:goodmarg}.

\begin{algorithm}[h]
	\caption{Graph Filter}\label{alg:newforward}
	\begin{algorithmic}[1]
	\Require{ $\mathcal{K},	(N^m_f(K))_{K \in \mathcal{K}}, (N^m_v(K))_{K \in \mathcal{K}}, (\mu_0^K)_{K \in \mathcal{K}}, (p^v(\cdot,\cdot))_{v \in V},(g^f(\cdot, \cdot))_{f \in F}, (y_t)_{t=\{1,\dots,T\}}$}
	\For{$K \in \mathcal{K}$}
		\State{$\tilde{\pi}_0^K\leftarrow  \mu_0^K$}
		\State{\bf{Compute  $p^K(\cdot,\cdot)\leftarrow \prod_{v \in K} p^v(\cdot,\cdot)$}}
	\EndFor
	\For{$t \in \{1,\dots,T\}$}
		\For{$K \in \mathcal{K}$}
			\For{${z^K \in \mathbb{X}^K}$}
				\State{${\hat{\pi}}^K(z^K) \leftarrow \sum_{x^K \in \mathbb{X}^K} p^K(x^K,z^K) \cdot \tilde{\pi}^K_{t-1}(x^K)$}
			\EndFor
		\EndFor
		\State{$(\tilde{\pi}_t^K)_{K \in \mathcal{K}} \leftarrow$\bf{Algorithm \ref{alg:margcon}}$\left (\mathcal{K}, (N^m_f(K))_{K \in \mathcal{K}}, (N^m_v(K))_{K \in \mathcal{K}},(\hat{\pi}^K)_{K \in \mathcal{K}},(g^f(\cdot, y_t))_{f \in F} \right )$}
	\EndFor
	\State{\Return {$((\tilde{\pi}_t^K)_{K \in \mathcal{K}})_{t = \{0,\dots,T \}}$}}
	\end{algorithmic}
\end{algorithm}

Implementation of the approximate filtering method is shown in Algorithm \ref{alg:newforward}, which we shall refer to from now on as the Graph Filter. Noting that the complexity of computing $\hat{\pi}^K \leftarrow (\mathsf{P} \mu)^K$, which is $\mathcal{O}(T L^{2 \max\limits_{K \in \mathcal{K}}\mathbf{card}(K)})$, is dominated by the cost of Algorithm \ref{alg:margcon}, with an additional $2$ at the exponent of $L$. The overall complexity of Algorithm \ref{alg:newforward} is then:
$$
\mathcal{O} \left ( T \mathbf{card}(\mathcal{K}) \max\limits_{K \in \mathcal{K}}\mathbf{card}\left(N^m_f(K)\right) L^{2\max\limits_{K \in \mathcal{K}}\mathbf{card}(N^m_v(K))} \right ).
$$
Remark that when implementing the correction step the Bayes update is performed in the for loop over the partition. Even though such an implementation is faster, it requires a single loop over the partition, the theoretical computational cost does not change.

The reader is referred back to the end of section \ref{subsec:approx_bayes} for detailed discussion of the exponent of $L$ appearing here. We note that for the chain graph example in Figure \ref{fig:fgsimple} with $\mathcal{K}=\{\{1\},\{2\}, \ldots \}$, the complexity is:
\begin{equation*} \label{costtheorychain}
\mathcal{O} \left ( T \mathbf{card}(\mathcal{K}) \min\{2(m+1), M-1\} L^{2 \min\{2(m+1)+1, M\}}  \right ).
\end{equation*}
Thus the exponent of $L$ has the dimension free property of becoming independent of $M$ depending on $m$ when $M$ is large enough.

\subsection{Graph Smoother}\label{subsec:approx_smoothing}
The approximate smoothing distributions are defined by simply substituting the approximate filtering distributions into \eqref{smoothing}:
\begin{equation}\label{newsmoothing}
\tilde{\pi}_{T|T}\coloneq \tilde{\pi}_{T},\qquad
\tilde{\pi}_{t|T}\coloneq \mathsf{R}_{\tilde{\pi}_{t}} \tilde{\pi}_{t+1|T}, \qquad t \in \{T-1,\ldots,0\}.
\end{equation}

\begin{theorem} \label{thm:smoothing}
Fix any collection of observations $\{y_1,\dots,y_T\}$ and any partition $\mathcal{K}$ of $V$. There exists a region $\tilde{\mathcal{R}}_0 \subseteq (0,1)^3$  depending only on $\tilde{\Upsilon}$,$\Upsilon$ and $\Upsilon^{(2)}$, such that if, for given $(\epsilon_{-},\epsilon_{+},\kappa) \in \tilde{\mathcal{R}}_0$,
$$
\epsilon_{-} \leq p^v(x^v,z^v)  \leq {\epsilon_{+}},
\quad
\text{and}
\quad
\kappa \leq g^f\left ({x}^{N{(f)}},y_t \right ) \leq \frac{1}{\kappa},
$$
for all $x, z \in\mathbb{X}^V, v \in V, f \in F, t \in \{1,\dots,T\}$, then for $\beta > \log(2)$ small enough depending only on $\tilde{\Upsilon},\Upsilon,\Upsilon^{(2)}, \epsilon_{-},\epsilon_{+}$ and $\kappa$, we have that for any $\mu_0$ satisfying: 
\begin{equation}
\widetilde{\mbox{Corr}}(\mu_0,\beta) \leq 2 e^{-\beta} \left ( 1 - \frac{\epsilon_{-}}{\epsilon_{+}} \right ) + 2 e^{2 \beta} \Upsilon^{(2)} \left ( 1 - \kappa^{2 \tilde{\Upsilon}} \right )
\end{equation}
and for any $K \in \mathcal{K}$, $J \subseteq K$ and $m \in \{0,\dots,n\}$:
\begin{equation*}
\begin{split}
    \norm{\tilde{\pi}_{t|T} - {\pi}_{t|T}}_J
     & \leq
\alpha_2(\beta, \epsilon_{-}, \epsilon_{+})\left ( 1-\kappa^{a(\mathcal{K})} \right )\mathbf{card}(J)+ \gamma_2(\beta, \epsilon_{-}, \epsilon_{+}) \left ( 1-\kappa^{b(\mathcal{K},m)} \right ) \mathbf{card}(J) e^{-\beta m},
\end{split}
\end{equation*}
where $\pi_{t|T}, \tilde{\pi}_{t|T}$ are given by \eqref{smoothing} and \eqref{newsmoothing} with initial distribution $\mu_0$,  $\alpha_2(\beta, \epsilon_{-},\epsilon_{+})$ and $\gamma_2(\beta,\epsilon_{-},\epsilon_{+})$ are constants depending on $\epsilon_{-},\epsilon_{+}$, $\beta$ and
\begin{align*}
a(\mathcal{K}) &\coloneq 2 \max_{K \in \mathcal{K}} \max_{v \in \partial K} \mathbf{card}(N(v) \cap \partial N(K)),\\
b(m,\mathcal{K}) &\coloneq 2 \max_{K \in \mathcal{K}}\max_{v \notin { N^{m-1}_v(K)}}  \mathbf{card}(N(v)),
\end{align*}
with the convention that the maximum over an empty set is zero.
\end{theorem}
The proof of Theorem \ref{thm:smoothing} is in appendix A. The only difference between the bound in this theorem and that in Theorem \ref{th:goodapprox} are the constants $\alpha_2(\beta, \epsilon_{-},\epsilon_{+})$ and $\gamma_2(\beta,\epsilon_{-},\epsilon_{+})$, explicit expressions can be deduced from the proof. The full assumption on the initial distribution $\mu_0$ can be found in appendix A.

\begin{algorithm}[h]
\caption{Graph Smoother}\label{alg:newback}
\begin{algorithmic}[1]
\Require{ $\mathcal{K}, (p^v(\cdot,\cdot))_{v \in V}, (\tilde{\pi}_t^K)_{K \in \mathcal{K},t = \{0,\dots,T \}}$}
  \For{$K \in \mathcal{K}$}
 	\State{ $\tilde{\pi}_{T|T}^K \leftarrow \tilde{\pi}_T^K$}
 	\State{\bf{Compute } $p^K(\cdot,\cdot)\coloneq \prod_{v \in K} p^v(\cdot,\cdot)$}
 	\EndFor
 \For{$t \in \{T-1,\dots,0\}$}
		\For{$K \in \mathcal{K}$}
		\For{$z^K \in \mathbb{X}^K$}
		\For{$x^K \in \mathbb{X}^K$}
		\State{$\overleftarrow{p}^K(z^K,x^K) \leftarrow {p^K(x^K,z^K) \tilde{\pi}_{t}^K({x^K})}$}
		\EndFor
		\State{\bf{Normalize $\overleftarrow{p}^K(z^K,\cdot)$ to a probability mass function on} $\mathbb{X}^{{K}}$ }
		\EndFor
		\For{$x^K \in \mathbb{X}^K$}
		\State{$\tilde{\pi}_{t|T}^K(x^K) \leftarrow \sum_{z^K \in \mathbb{X}^K} \overleftarrow{p}^K(z^K,x^K)\tilde{\pi}_{t+1|T}^K({z^K})$}
		\EndFor
		\EndFor
 \EndFor	
\State{\Return {$((\tilde{\pi}_{t|T}^K)_{K \in \mathcal{K}})_{t = \{1,\dots,T \}}$}}
\end{algorithmic}
\end{algorithm}

The approximate smoothing method, shown in Algorithm \ref{alg:newback} has complexity
$$
\mathcal{O} \left ( T\mathbf{card}(\mathcal{K})L^{3 \max\limits_{K \in \mathcal{K}}\mathbf{card}(K)} \right ).
$$

Assuming $3 \max_{K \in \mathcal{K}}\mathbf{card}(K)$ is smaller than $2 \max_{K \in \mathcal{K}}\mathbf{card}(N^m_v(K))$, which is typically the case in practice, the overall complexity of Algorithm \ref{alg:newforward} combined with Algorithm \ref{alg:newback} is:
$$
\mathcal{O} \left ( T \mathbf{card}(\mathcal{K}) \max\limits_{K \in \mathcal{K}}\mathbf{card}\left(N^m_f(K)\right) L^{2\max\limits_{K \in \mathcal{K}}\mathbf{card}(N^m_v(K))} \right ).
$$
The reader is referred back to the end of section \ref{subsec:approx_bayes} for detailed discussion of the exponent of $L$ appearing here.

\section{Numerical results}\label{sec:numerical}

Section \ref{subsec:model_spec} describes a class of FHMMs with conditionally Gaussian observations used as a running example in \cite{Ghahramani1997} and which we shall use in our numerical experiments. The purpose of the first set of experiments, in Section \ref{subsec:error_and_speed}, is to illustrate the practical implications of our theoretical results, assessing the performance of the Graph Filter and Smoother methods against exact filtering and smoothing, both in terms of accuracy and computational speed. In Section \ref{sub:comparison_VB} we compare the performance of EM algorithms for parameter estimation built around the Graph Smoother and variational approximations presented in \cite{Ghahramani1997}. In Section \ref{subsec:LU} we outline a model of traffic flow on the London Underground and illustrate parameter estimation and prediction using the Graph Filter and Smoother.

We used the University of Bristol's BlueCrystal High Performance Computing machine. The experiments were run on either one or two standard compute nodes each with {\sc{2 x 2.6GHz 8-core Intel E5-2670 (SandyBridge)}} chips and 4GB of RAM per core.

\subsection{Synthetic data}\label{subsec:model_spec}
With $\mathbb{X}$ a finite subset of $\mathbb{Z}$, $V=\{1,\ldots,M\}$ and $\mathbb{Y}=\mathbb{R}^{d_y}$, consider the Gaussian emission model from  \cite{Ghahramani1997}:
$$
g(x,y)= |\Sigma|^{-\frac{1}{2}} (2 \pi)^{-\frac{d_y}{2}} \exp \left \{ -\frac{1}{2} \left [ y-a(x) \right ]^T \Sigma^{-1} \left [ y-a(x) \right] \right \},
$$
where $a(x)$ is a vector whose entries may depend on $x$.
\begin{figure}[httb!]
\centering
\resizebox{1\textwidth}{0.15\textheight}{%
\begin{tikzpicture}[on grid ,state/.style ={ circle , draw , text=blue,font=\bfseries\LARGE , minimum width =2 cm}]
\node[state] (Xt1) at (0,0) {$X_t^{5}$};
\node[state] (Xt2) at (-6,0) {$X_t^{4}$};
\node[state] (Xt3) at (-12,0) {$X_t^{3}$};
\node[state] (Xt4) at (-18,0) {$X_t^{2}$};
\node[state] (Xt5) at (-24,0) {$X_t^{1}$};


\draw[thick] (1,4.5)   -- (1,4.5)   node[red] {$Y_t^{4} \sim \mbox{N}(cX_t^{5}+cX_t^{4},\sigma^2)$};
\draw[thick] (-8,4.5)  -- (-8,4.5)  node[red] {$Y_t^{3} \sim \mbox{N}(cX_t^{4}+cX_t^{3},\sigma^2)$};
\draw[thick] (-16,4.5) -- (-16,4.5) node[red] {$Y_t^{2} \sim \mbox{N}(cX_t^{3}+cX_t^{2},\sigma^2)$};
\draw[thick] (-25,4.5) -- (-25,4.5) node[red] {$Y_t^{1} \sim \mbox{N}(cX_t^{2}+cX_t^{1},\sigma^2)$};


\draw[thick] (0,-1.75)   -- (0,-1.75)   node[blue] {$X_t^{5} \sim p(X_{t-1}^{5}, \cdot)$};
\draw[thick] (-6,-1.75)  -- (-6,-1.75)   node[blue] {$X_t^{4} \sim p(X_{t-1}^{4}, \cdot)$};
\draw[thick] (-12,-1.75)  -- (-12,-1.75)  node[blue] {$X_t^{3} \sim p(X_{t-1}^{3}, \cdot)$};
\draw[thick] (-18,-1.75) -- (-18,-1.75) node[blue] {$X_t^{2} \sim p(X_{t-1}^{2}, \cdot)$};
\draw[thick] (-24,-1.75) -- (-24,-1.75) node[blue] {$X_t^{1} \sim p(X_{t-1}^{1}, \cdot)$};

\node[state, diamond, text=red] (f1) at (0,3) {$f^4$};
\node[state, diamond, text=red] (f2) at (-9,3) {$f^3$};
\node[state, diamond, text=red] (f3) at (-15,3) {$f^2$};
\node[state, diamond, text=red] (f4) at (-24,3) {$f^1$};

\path (Xt1) edge [red] node {} (f1);
\path (Xt2) edge [red] node {} (f1);
\path (Xt2) edge [red] node {} (f2);
\path (Xt3) edge [red] node {} (f2);
\path (Xt3) edge [red] node {} (f3);
\path (Xt4) edge [red] node {} (f3);
\path (Xt4) edge [red] node {} (f4);
\path (Xt5) edge [red] node {} (f4);

\end{tikzpicture}
,}
\caption{Factor graph for the model \eqref{ass:simplified} in the case $F=\{f^1,f^2,f^3,f^4\}$ and $V=\{1,2,3,4,5\}$.} \label{fig:experimental1}
\end{figure}

We specialize to the case $d_{y}=M-1$ and the specific forms of $\Sigma$ and $a(x)$:
\begin{equation} \label{ass:simplified}
\Sigma= \sigma^2 I \quad \text{and} \quad a(x) =  \left ( a^{f}(x) \right )_{f \in \{1 \dots M-1\}} \quad \text{with} \quad a^f (x)\coloneq c \left (x^{f}+ x^{f+1} \right),
\end{equation}
where $c>0$ is a constant. Under these assumptions, $N(f)=\{f,f+1\}$ and
\begin{equation}
g^f(x^{N(f)},y)=\frac{1}{\sqrt{2\pi}\sigma}\exp\left\{-\frac{[y^f-c(x^{f}+ x^{f+1})]^2}{2\sigma^2}\right\}.\label{eq:g^f_example}
\end{equation}

The corresponding factor graph $\mathcal{G}$ is a chain, illustrated in Figure \ref{fig:experimental1}. We also assume that the transition probabilities and initial distribution have identical components across $V$,
\begin{equation}
p(x,z)= \prod_{v \in V} p^v(x^v,z^v),\quad
\mu_0(x)= \prod_{v \in V} \mu_0^v(x^v),\quad p^v= \hat{p},\;\mu_0^v = \hat{\mu}_0,\;\forall v\in V.
\end{equation}
Throughout the experiments we fix the partition $\mathcal{K}$ as:
$$
\mathcal{K} = \{ \{1\},\dots,\{M\} \}.
$$

\paragraph{Accuracy and speed performance for filtering and smoothing}\label{subsec:error_and_speed}
We took  $\mathbb{X}= \{0,1\}$ and simulated three data sets of length $T=500$ from the model with parameters:
\begin{equation}\label{eq:ex1_params}
\hat{\mu}_0(x^v)= 1,\;x^v=1, \forall v\in V, \quad
\{\hat{p}(x^v,z^v)\}_{x^v,z^v \in \mathbb{X}}=
\begin{pmatrix}
0.6 & 0.4 \\
0.2 & 0.8
\end{pmatrix}, \quad
c=1, \quad
\sigma^2 =1.
\end{equation}

\begin{figure}[httb!]
\centering
\begin{subfigure}[t]{1\textwidth}
\includegraphics[trim={1cm 1cm 4cm 1cm},scale=0.3]{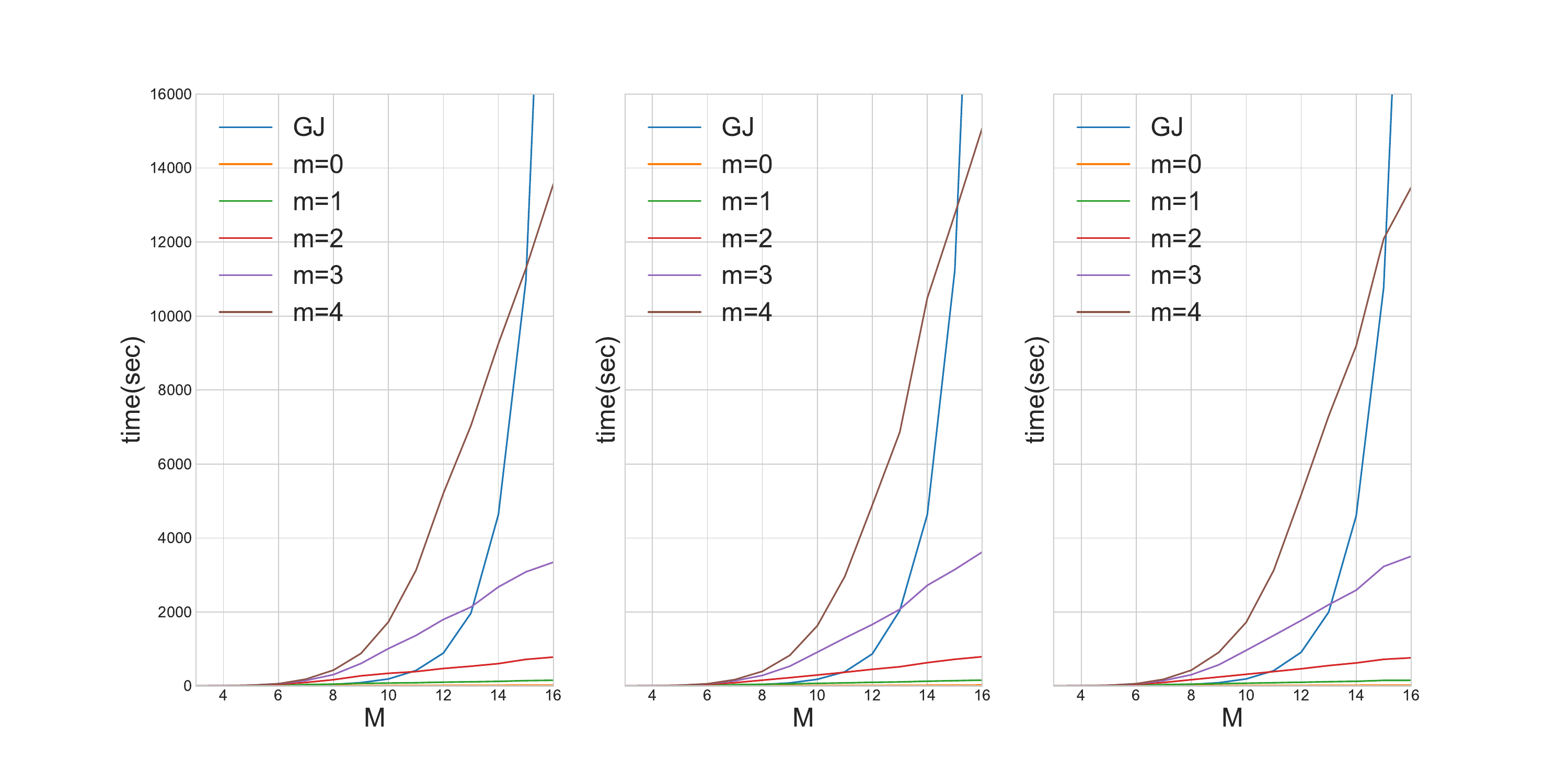}
\end{subfigure}
\begin{subfigure}[t]{1\textwidth}
\includegraphics[trim={1cm 1cm 4cm 1cm},scale=0.3]{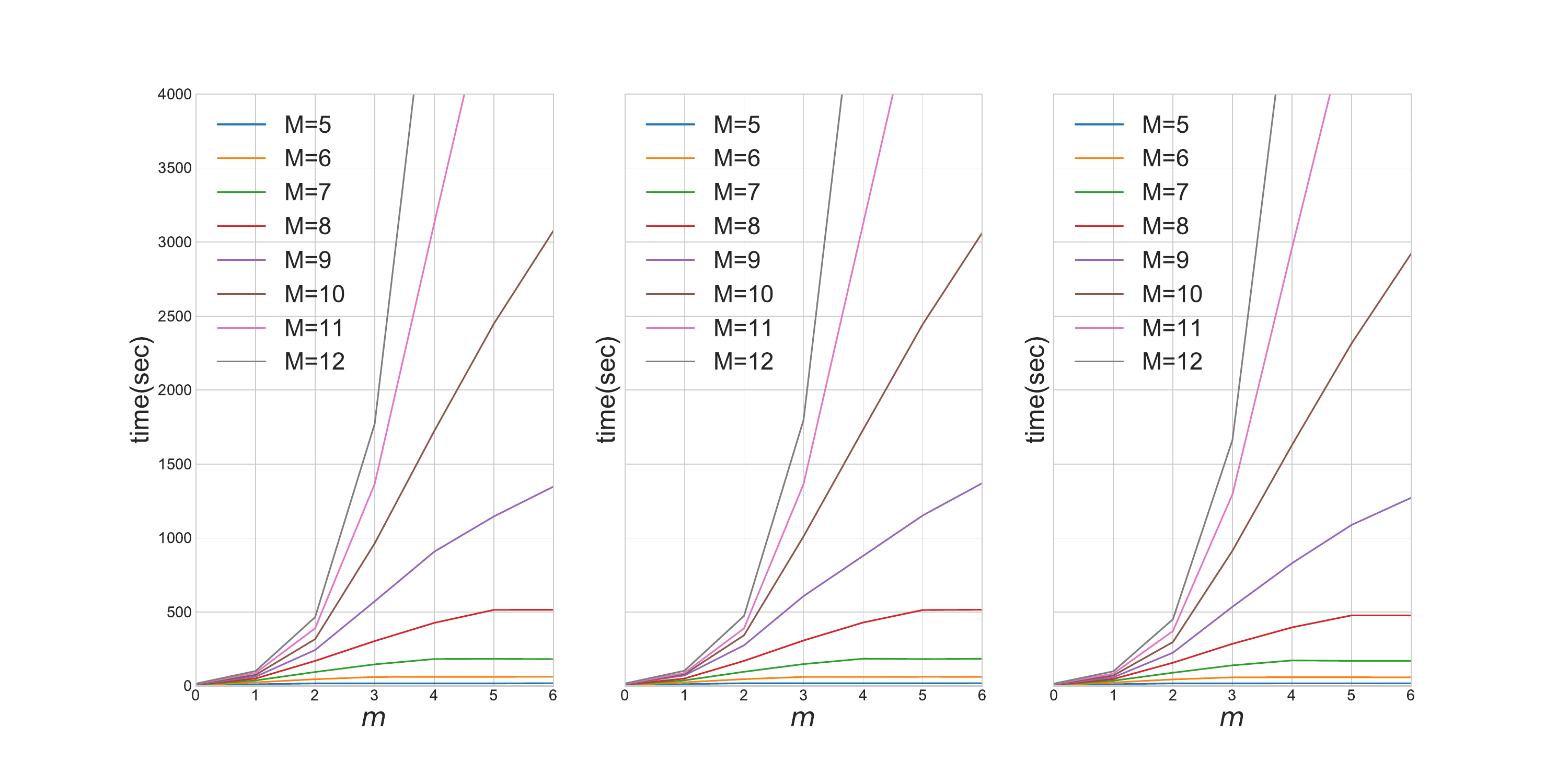}
\end{subfigure}
\caption{Execution time for the combined filtering and smoothing algorithms as a function of $m$ and $M$. Each vertical pair of plots corresponds to one of three simulated data sets. GJ is the exact filtering and smoothing algorithm of \cite{Ghahramani1997}.} \label{fig:experimentaltime2}
\end{figure}

First consider the execution time of the approximate filtering and smoothing method, i.e., the combination of Algorithm \ref{alg:newforward} and Algorithm \ref{alg:newback}, as a function of the parameters $m$ and $M$. Figure \ref{fig:experimentaltime2} shows execution time as $m$ and $M$ vary. The execution time of exact filtering and smoothing using the algorithm of \cite{Ghahramani1997}, henceforth \virg{GJ}, is included for reference.

It is apparent from the top row of plots that with $m$ fixed, the execution time of the Graph Filter and Smoother initially increases super-linearly with $M$ up to some point which depends on $m$, and from then on it is linear in $M$. This is most visually evident for the large values of $m$ and is consistent with the complexity of the combined Graph Filter and Smoother method discussed in Section \ref{subsec:approx_smoothing}, which for the model considered here is:
\begin{equation} \label{cost}
\mathcal{O} \left ( T M \min\{2(m+1), M-1\} L^{2 \min\{2(m+1)+1, M\}} \right ).
\end{equation}
By contrast, the execution time of GJ increases exponentially with $M$, making its implementation extremely expensive in high-dimensional cases.

When $M$ is fixed, it is clear from the bottom row of plots in Figure \ref{fig:experimentaltime2} that the execution time of the Graph Filter and Smoother is super-linear in $m$ up to some point which depends on $M$, and then is constant in $m$. Again this is consistent with \eqref{cost}. The phenomenon of the cost becoming constant in $m$ arises because as $m$ grows, eventually all factors are included in the products in $\tilde{\mathsf{C}}^{m,K}_t$, see \eqref{eq:localizfact}.

\begin{figure}[h!]
\centering
\includegraphics[trim={4cm 1cm 4cm 1cm},clip,scale=0.33]{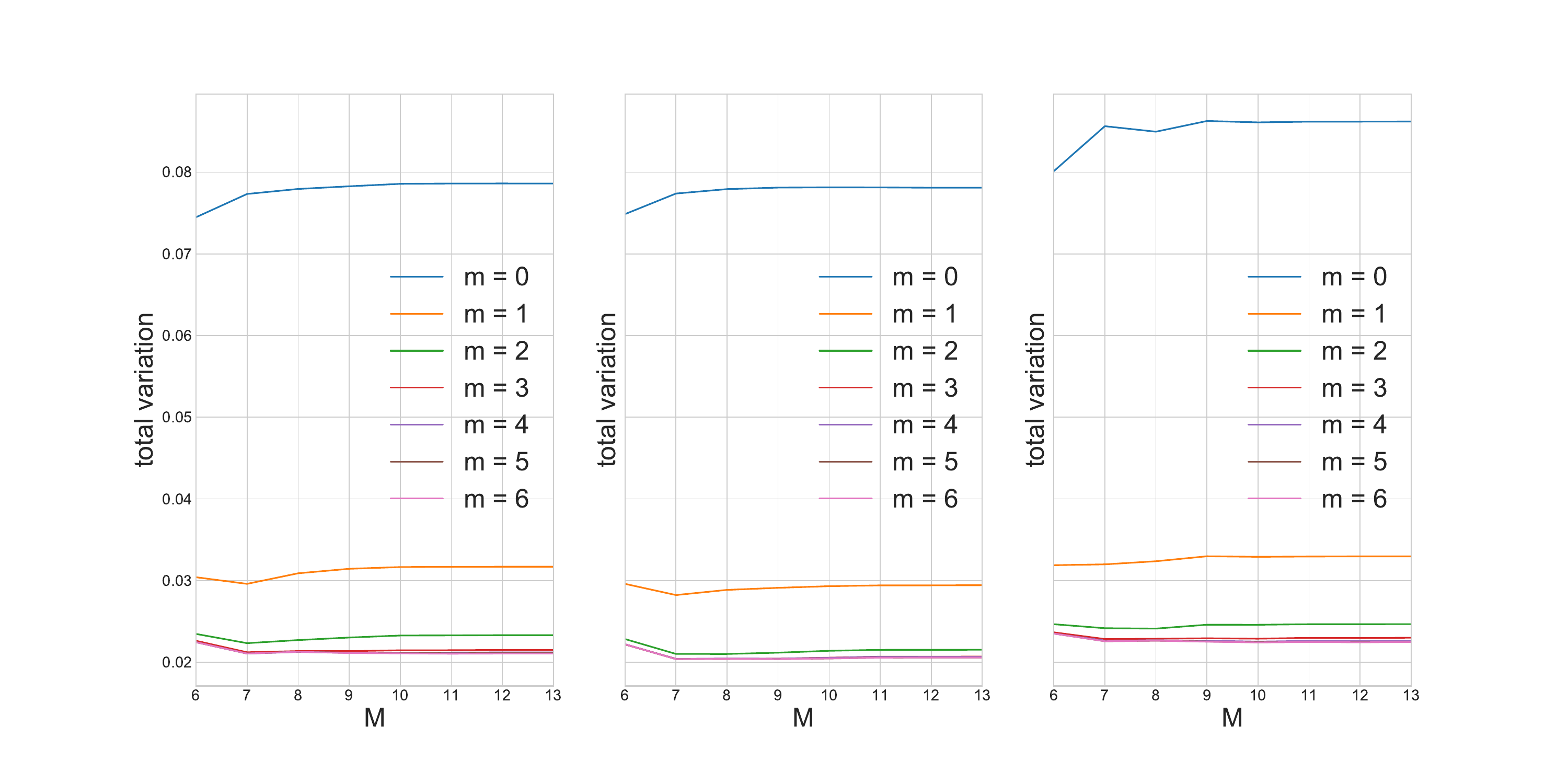}
\caption{The LTV distance between the approximate and exact marginal smoothing distributions averaged over both the components $X_t^i,\ldots,X_t^{i+4}$, with $i=1,\dots,8$, and the $T=500$ time steps. With $m$ fixed the average LTV is constant in $M$ for $M$ large enough. The three plots correspond to the three simulated data sets.} \label{fig:experimentaltime3}
\end{figure}

We now examine accuracy. Recall two important characteristics of the bound of Theorem \ref{thm:smoothing}: the bound does not depend on the overall dimension, $M$, and decays exponentially with $m$. The region $\tilde{\mathcal{R}}_0$ in Theorem \ref{thm:smoothing} is non-empty,  but  for the specific parameter settings in \eqref{eq:ex1_params} there does not exist $(\epsilon_-, \epsilon_+,\kappa)\in\tilde{\mathcal{R}}_0$ such that the assumptions of the theorem on $p^v$ and $g^f$ hold.   Thus technically Theorem \ref{thm:smoothing} does not hold in this example. However, Figure \ref{fig:experimentaltime3} and Figure \ref{fig:experimentaltime4} encouragingly show that the LTV between the exact and approximate smoothing distributions exhibits the characteristics of not depending on the overall dimension, $M$, and decaying exponentially with $m$.

\begin{figure}[httb]
\centering
\includegraphics[trim={4cm 1cm 4cm 1cm},clip,scale=0.33]{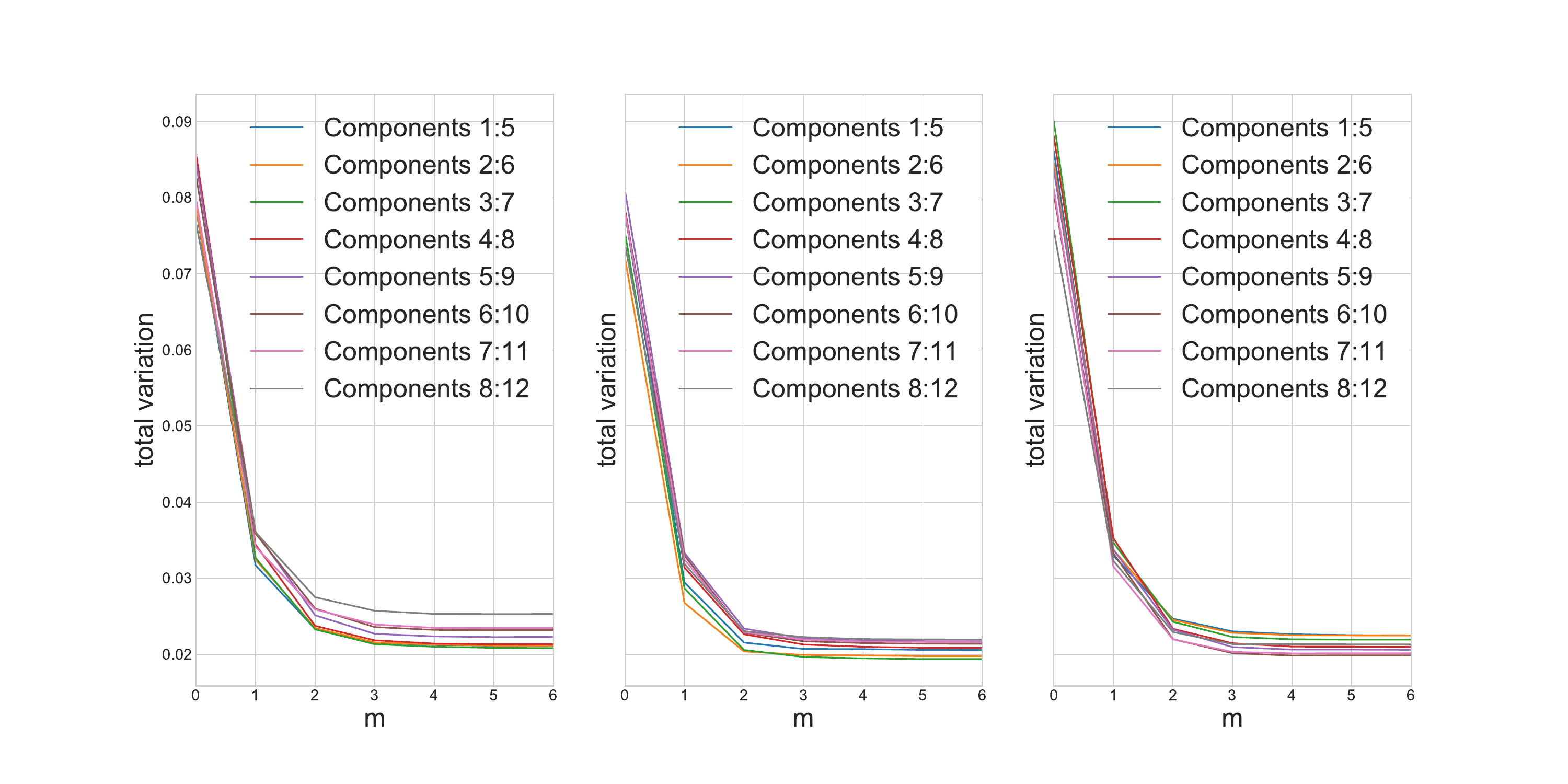}
\caption{The LTV distance between the approximate and exact marginal smoothing distributions of certain components of $X_t^1,\ldots,X_t^M$ and the $T=500$ time steps. The LTV decreases exponentially with $m$. The three plots correspond to the three simulated data sets.} \label{fig:experimentaltime4}
\end{figure}

\paragraph{Comparison to variational inference within EM for parameter estimation}\label{sub:comparison_VB}

Our next objective is to illustrate the accuracy of parameter estimation using the Graph Smoother within an approximate EM algorithm. We shall compare performance to the approximate EM approach of \cite{Ghahramani1997} in which variational approximations to the smoothing distributions are employed. For background on EM see \cite{dempster1977maximum} and \cite{Ghahramani1997} (Section 3.1).

\cite{Ghahramani1997} (Sections 3.4 and 3.5) describe two families of variational distributions for FHMM which can be used to compute the E-step in EM approximately: a \virg{fully-factorized} scheme in which the variational distribution is chosen to statistically decouple all state variables, $(X_t^v)_{v\in V}$, $t=0,\ldots, T$, in the HMM, and a \virg{structured} approximation, in which the variational distribution is Markovian in time but statistically decouples state variables across $V$. We shall refer to the former as \emph{completely decoupled} and the latter as \emph{spatially decoupled}.

The time complexity of computing the approximate smoothing distributions using either the completely decoupled or spatially decoupled schemes is:
\begin{equation}\label{eq:VB_cost}
\mathcal{O}(ITL^2M^2(M-1)^4),
\end{equation}
where $I$ is the number of iterations of the fixed-point equations needed to find the variational approximation. In our experiments we found that $I=20$ was sufficient for convergence, indeed \cite{Ghahramani1997} (pag. 254) suggest 2-10 iterations is typically sufficient. Recall that for $M$ large enough, \eqref{cost} is exponential in $m$, but linear in $M$, while \eqref{eq:VB_cost} scales no faster than $M^6$. Whether or not the variational approximations can be computed more quickly than the Graph Smoother is dependent on the model in question. In our experiments we did not find a substantial difference in speed.

Details of the EM updates using the Graph Smoother are given in appendix B. The only difference between these updates and those using the variational approximations is in the E-step, where the expectation is simply taken with respect to the corresponding approximate smoothing distribution.

\begin{table}[httb!]
	\centering
	\resizebox{\columnwidth}{!}{
		\begin{tabular}{cccccccc}
			\hline
			\hline
			Method                          & $\mu_0(0), \mu_0(1)$ & $c$   & $\sigma^2$ & $p(0,0), p(0,1), p(1,0), p(1,1)$  \\
			\hline
			\hline
			True values                     &     0.000, 1.000     & 2.000 & 4.000      &       0.600, 0.400, 0.200, 0.800 \\
			\hline
			Graph Smoother $m=0$            &     0.137, 0.863     & 1.753 & 4.642      &       0.072, 0.928, 0.518, 0.482 \\
			\hline
			Graph Smoother $m=1$            &     0.001, 0.999     & 1.779 & 4.542      &       0.075, 0.925, 0.549, 0.451 \\
			\hline
			Completely decoupled Variational Bayes &     0.384, 0.616     & 1.498 & 6.562      &       0.075, 0.925, 0.084, 0.916 \\
			\hline
			Spatially decoupled Variational Bayes  &     0.393, 0.607     & 1.960 & 6.223      &       0.362, 0.638, 0.385, 0.615 \\
			\hline
		\end{tabular}
	}
	\caption{Parameters estimates for the case $M=3$ with Graph Filter-Smoother and variational Bayes at the end of the EM algorithm. The estimates are found by taking the mean over the different initializations.}
	\label{tab:parametres_values_1}
\end{table} 

\begin{figure}[httb!]
	\centering
	\begin{subfigure}[t]{1 \textwidth}
		\includegraphics[trim={9cm 1cm 9cm 1cm},clip,scale=0.18]{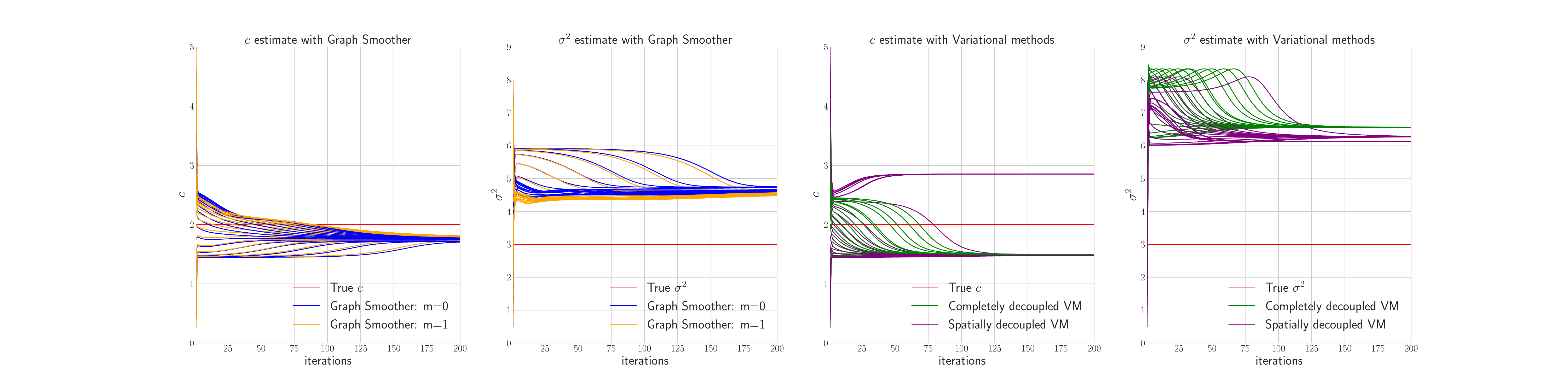}
	\end{subfigure}
	\caption{$M=3$. Estimation of $c$ and $\sigma^2$ using approximate EM based on the Graph Smoother and the completely and spatially decoupled variational approximations by \cite{Ghahramani1997}. Horizontal axes correspond to EM iterations. 20 different EM initializations shown for each algorithm setting.} \label{fig:experimentaltime5par}
\end{figure}

\begin{figure}[httb!]
	\begin{subfigure}[t]{1\textwidth}
		\includegraphics[trim={9cm 1cm 9cm 1cm},clip,scale=0.18]{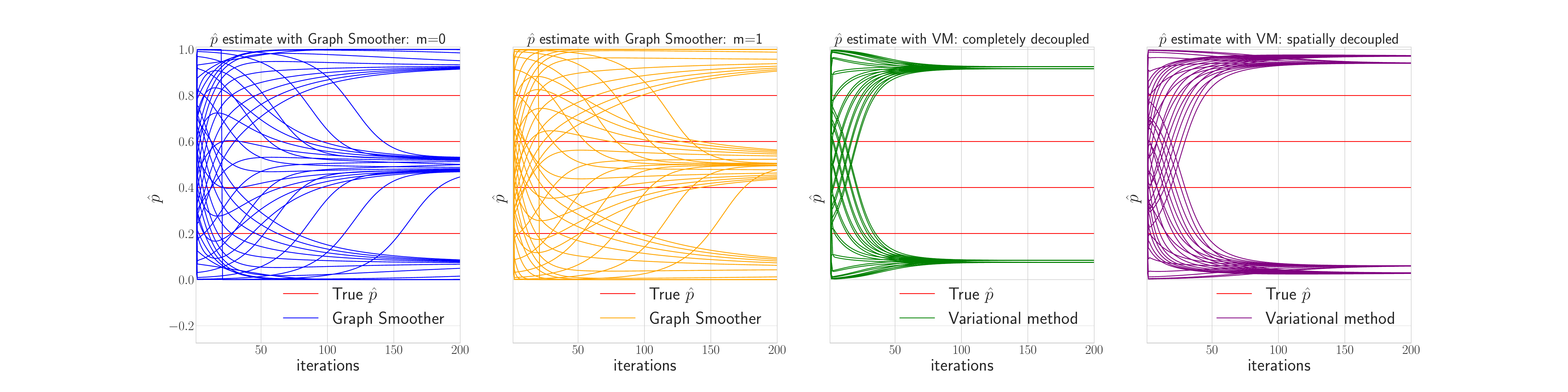}
	\end{subfigure}
	\begin{subfigure}[t]{1 \textwidth}
		\includegraphics[trim={9cm 1cm 9cm 1cm},clip,scale=0.18]{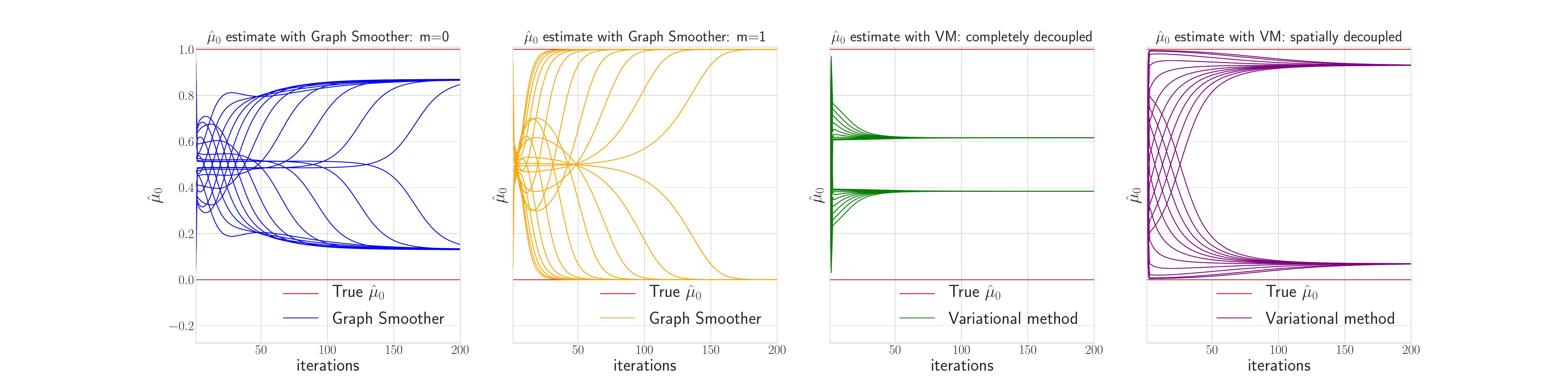}
	\end{subfigure}
	\caption{$M=3$. Estimation of $\hat{\mu}_0$ and $\hat{p}$ using approximate EM based on the Graph Smoother and the completely and spatially decoupled variational approximations by \cite{Ghahramani1997}. Horizontal axes correspond to EM iterations. 20 different EM initializations shown for each algorithm setting. Traces corresponding to the 2 elements of the initial distribution $\hat{\mu}_0$ and 4 elements of the transition matrix $\hat{p}$ are superimposed on each plot.} \label{fig:experimentaltime5kernel}
\end{figure}

\begin{table}[httb!]
	\centering
	\resizebox{\columnwidth}{!}{
		\begin{tabular}{cccccccc}
			\hline
			\hline
			Method                                 & $\mu_0(0), \mu_0(1)$ & $c$   & $\sigma^2$ & $p(0,0), p(0,1), p(1,0), p(1,1)$ \\
			\hline
			\hline
			True values                            &     0.000, 1.000     & 2.000 & 4.000      &       0.600, 0.400, 0.200, 0.800\\
			\hline
			Graph Smoother $m=0$                   &     0.206, 0.794     & 1.740 & 4.790      &       0.492, 0.508, 0.094, 0.906 \\
			\hline
			Graph Smoother $m=1$                   &     0.496, 0.504     & 1.765 & 4.651      &       0.556, 0.444, 0.096, 0.904 \\
			\hline
			Completely decoupled Variational Bayes &     0.469, 0.531     & 1.833 & 6.260      &       0.343, 0.657, 0.336, 0.664 \\
			\hline
			Spatially decoupled Variational Bayes  &     0.026, 0.974     & 1.423 & 5.804      &       0.044, 0.956, 0.040, 0.960 \\
			\hline
		\end{tabular}
	}
	\caption{Parameters estimates for the case $M=10$ with  Graph Filter-Smoother and variational Bayes at the end of the EM algorithm. The estimates are found by taking the mean over the different initializations.}
	\label{tab:parametres_values_2}
\end{table}

\begin{figure}[h!]
\centering
\begin{subfigure}[t]{1 \textwidth}
\includegraphics[trim={9cm 1cm 9cm 1cm},clip,scale=0.18]{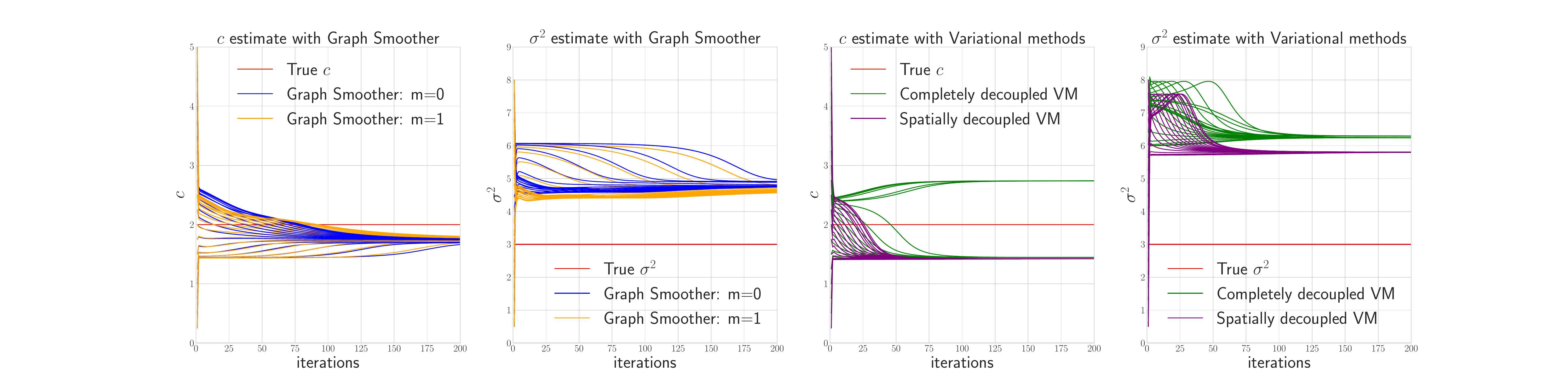}
\end{subfigure}
\caption{$M=10$. Estimation of $c$ and $\sigma^2$ using approximate EM based on the Graph Smoother and the completely and spatially decoupled variational approximations by \cite{Ghahramani1997}. Horizontal axes correspond to EM iterations. 20 different EM initializations shown for each algorithm setting.} \label{fig:experimentaltime6par}
\end{figure}

\begin{figure}[httb!]
\begin{subfigure}[t]{1\textwidth}
\includegraphics[trim={9cm 1cm 9cm 1cm},clip,scale=0.18]{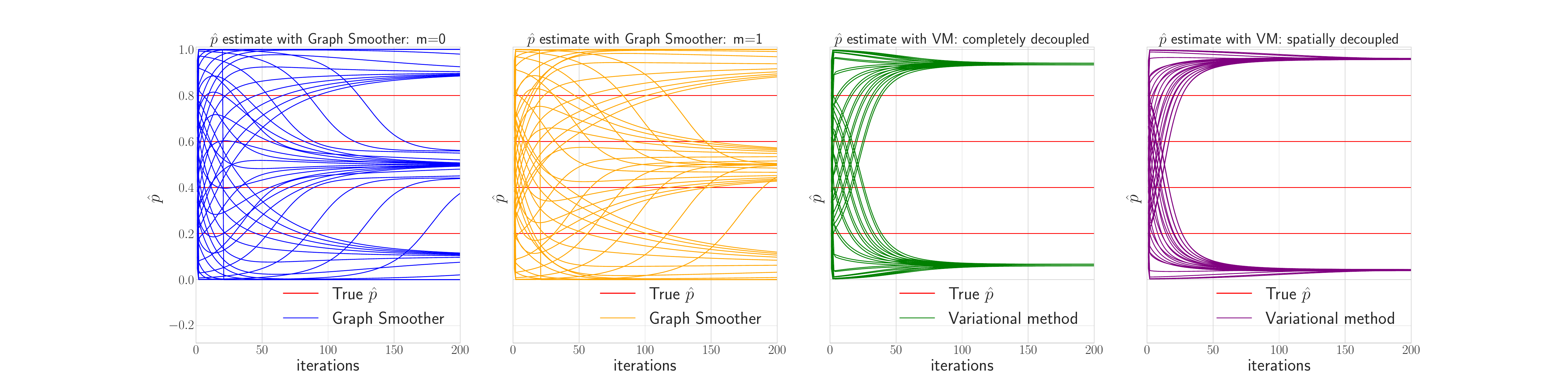}
\end{subfigure}
\begin{subfigure}[t]{1 \textwidth}
\includegraphics[trim={9cm 1cm 9cm 1cm},clip,scale=0.18]{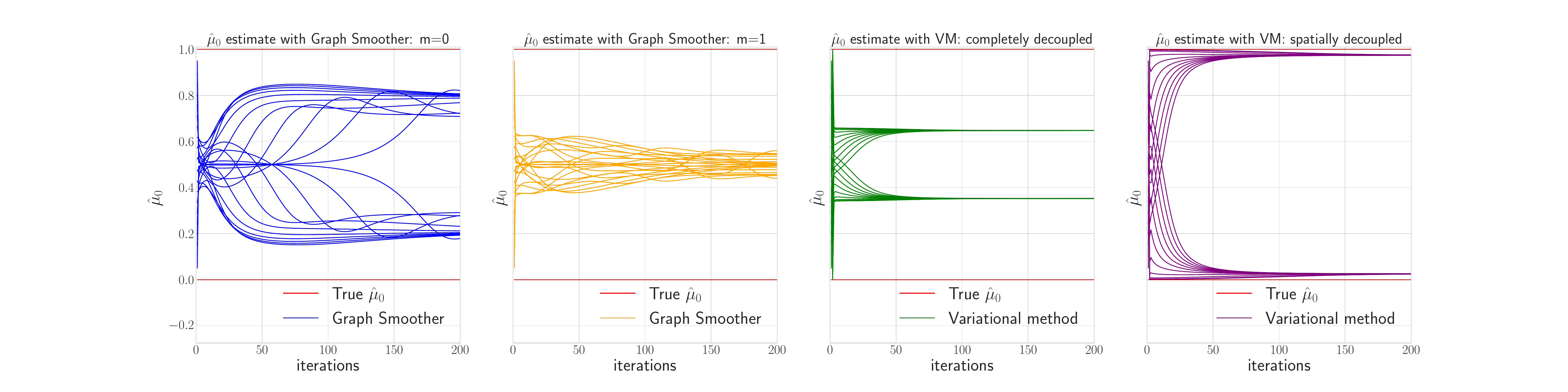}
\end{subfigure}
\caption{$M=10$. Estimation of $\hat{\mu}_0$ and $\hat{p}$ using approximate EM based on the Graph Smoother and the completely and spatially decoupled variational approximations by \cite{Ghahramani1997}. Horizontal axes correspond to EM iterations. 20 different EM initializations shown for each algorithm setting. Traces corresponding to the 2 elements of the initial distribution $\hat{\mu}_0$ and 4 elements of the transition matrix $\hat{p}$ are superimposed on each plot.} \label{fig:experimentaltime6kernel}
\end{figure}

In our experiments we considered the model described in Section \ref{subsec:model_spec}, with $\mathbb{X}=\{0,1\}$ and $T=200$, and generated a data set with true parameter values:
$$
\hat{\mu}_0(x^v)= 1,\;x^v=1, \forall v\in V, \quad
\{\hat{p}(x^v,z^v)\}_{x^v,z^v \in \mathbb{X}}=
\begin{pmatrix}
0.6 & 0.4 \\
0.2 & 0.8
\end{pmatrix}, \quad
c=2 \quad
\text{and} \quad
\sigma^2 =4.
$$
The EM algorithms based on the Graph Smoother and the fully and spatially decoupled variational approximations were run for $20$ different EM initializations. Tables \ref{tab:parametres_values_1} and \ref{tab:parametres_values_2} show the mean of the final estimates over the different EM initializations for $M=3$ and $M=10$, respectively. Graphical results are presented in figures \ref{fig:experimentaltime5par} and \ref{fig:experimentaltime5kernel} for $M=3$ and figures \ref{fig:experimentaltime6par} and \ref{fig:experimentaltime6kernel} for $M=10$.

In Figure \ref{fig:experimentaltime5par} the EM algorithms associated with the Graph Smoother converge to points closer to the true parameter values than those using the variational approximations. Using $m=1$ rather than $m=0$ in the former yields a slight increase in accuracy. In Figure \ref{fig:experimentaltime5kernel} the results using the Graph Smoother are not more accurate in all cases, but the completely decoupled variational method generally performs badly. In Figure  \ref{fig:experimentaltime6par} the results for the Graph Smoother are again more accurate.  In Figure \ref{fig:experimentaltime6kernel} it is notable that the estimates of the transition probabilities are a little more accurate with $m=0$ rather than $m=1$, but substantially more accurate than with either of the variational schemes. Again the completely decoupled variational approximation performs poorly. These observations are numerically evident also from tables \ref{tab:parametres_values_1} and \ref{tab:parametres_values_2}.

\subsection{Analyzing traffic flows on the London Underground}\label{subsec:LU}

Transport For London, the operator of the London Underground, has made publicly available \virg{tap} data, consisting of a 5\% sample of all Oyster card journeys in a week during November 2009 \citep{oyster}. The data consist of the locations and times of entry to and exit from the transport network for each trip.

Similar Transport for London data have be analyzed by \cite{silva2015predicting}, who developed models of numbers of trips between pairs of stations in the Underground network in order to quantify the effects of shocks such as line and station closures, and to predict traffic volumes. The modelling approach described by \cite{silva2015predicting}(Supporting Information) is very sophisticated, including several components such as regression of the numbers of passengers entering stations onto time, a cascade of nonparametric binomial models for the numbers of passengers inside the transport system who entered at each station and a Bayesian probabilistic flow model. One of many attractive features of this approach is that it avoids the computational cost of network tomography models for traffic data \citep{guimera2005worldwide,colizza2006role,newman2011structure} which is prohibitive in the context of large transport systems due to the exponential growth of problem size in the number of network links. Similar computational difficulties are encountered with some dynamic Bayesian network models of flow on transport networks. For example, \cite{hofleitner2012learning} propose a dynamic mixture model for travel times where the mixture component represents a time-varying congestion state associated with each link in the transport network. In principle, inference in this model can be performed using a particle filter, but as noted by \cite{woodard2017predicting}, due to high-dimensionality the cost of doing so accurately (with respect to Monte Carlo error) is very demanding.

It is not our objective to conduct as detailed modelling exercise as in these works, but rather to establish a proof of principle that the Graph Filter and Smoother are naturally suited to the topological structure of transport networks and show promise calibrating the model by estimating parameters, and for prediction. This leaves potential for a deeper investigation of traffic modelling using FHMMs in future work.


\begin{figure}[httb!]
	\centering
	\includegraphics[clip,width = \textwidth]{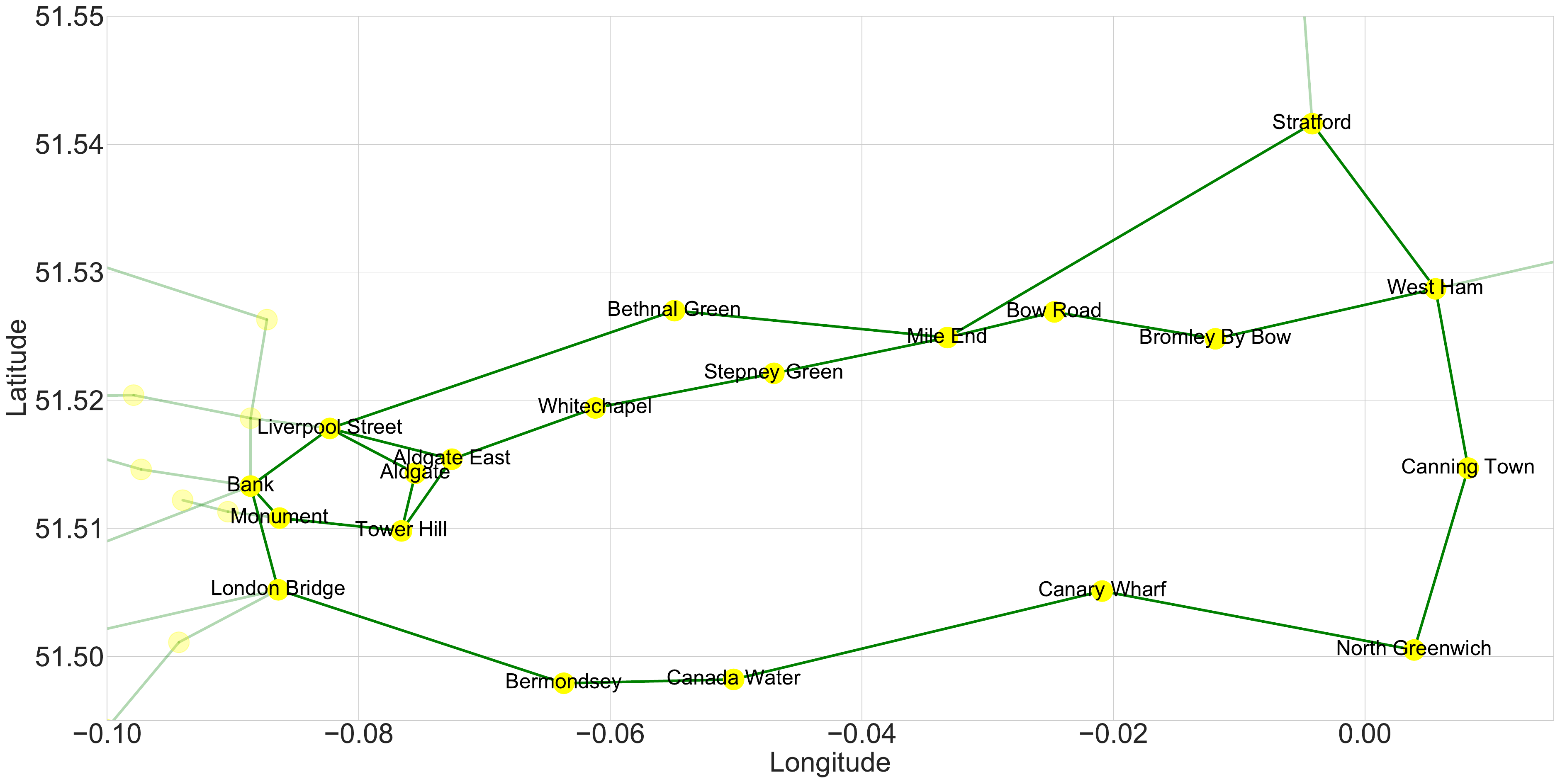}
	\caption{Locations in longitude and latitude of 20 stations from the the Central line and the Jubilee line of the London underground. Stations are represented as yellow nodes, while train lines are green edges.} \label{fig:londonnet}
\end{figure}

\subsubsection{Data and model} 
We consider 20 stations on a portion of the Central line and the Jubilee line.  Stations' names and geographical locations are shown in figure \ref{fig:londonnet}. The dataset consists of integer counts of the inflow and the outflow of passengers from Monday to Friday per station every $10$ minutes from 00:00 am to 00:00 am of the next day. The data are split into training given by Monday, Tuesday and Wednesday and test consisting of Thursday and Friday. 

In constructing our model we consider two factors per station, one for the inflow, the other for the outflow: for each $f \in \{1,\dots,20\}$, each corresponding to a station, we denote by $y_t^{f,\text{in}}$ and $y_t^{f,\text{out}}$ respectively the counts of passenger inflow and outflow at time $t$. We consider a hidden state variable associated with each direction of travel on each tube line segment connecting a pair of stations. These state variables are written $x_t^{i,j}$ to indicate the state of the line segment between stations $i$ and $j$ at time $t$ in the direction from $i$ to $j$. Thus the index set $V$ of the FHMM has one element corresponding to each direction of travel between each pair of stations which are connected by a single line segment on the tube network. The state-space of each variable $x_t^{i,j}$ is defined to be $\mathbb{X}=\{0,1,2,3\}$, with the interpretation of increasing levels of congestion on the line-segment and in the direction of travel corresponding to $(i,j)$.


\begin{figure}[h!]
	\centering
	\includegraphics[width = \textwidth]{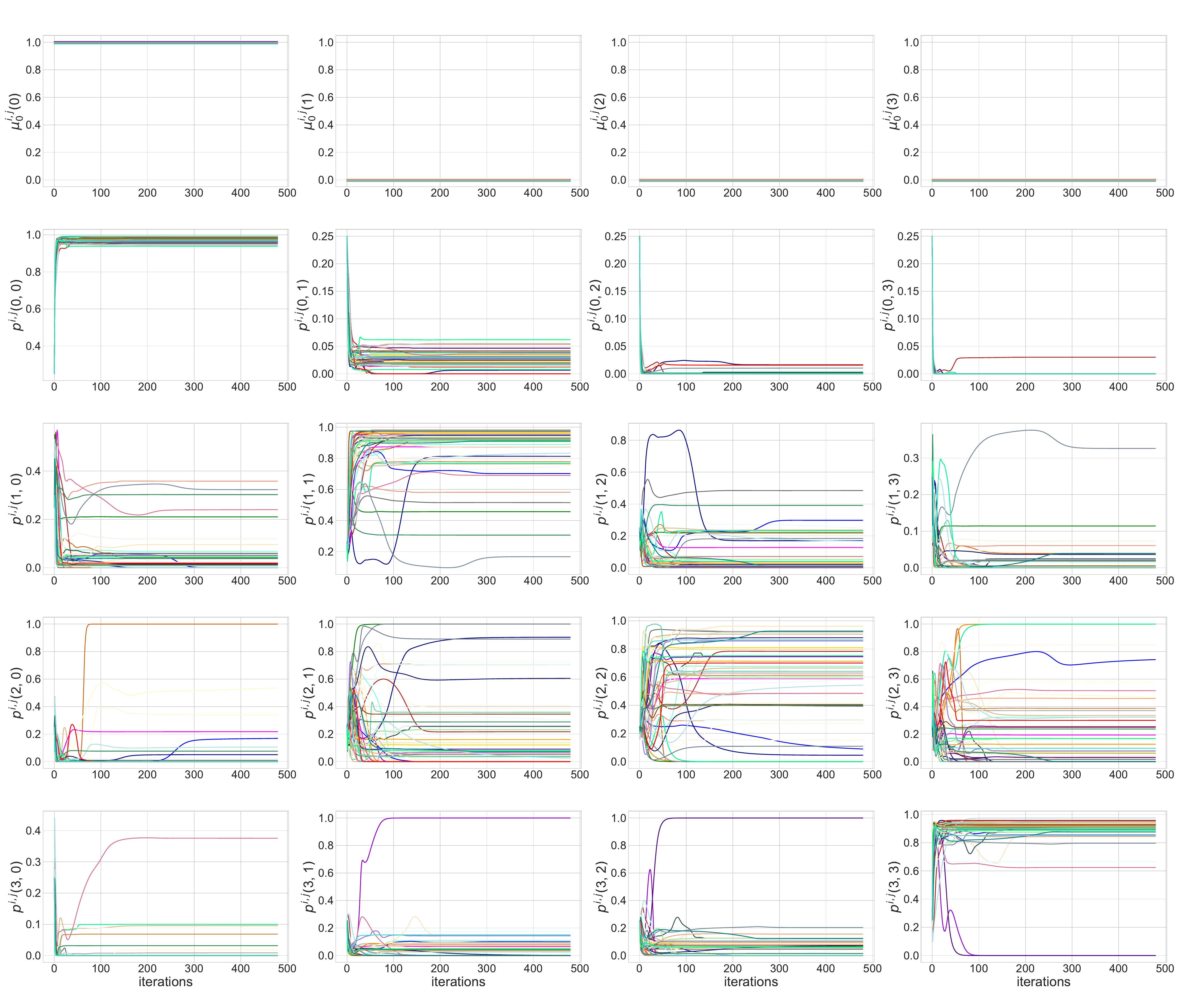}
	\caption{Estimation of the of initial distribution and transition matrix for each tube line using the approximate EM algorithm built around the Graph Filter-Smoother algorithm. See appendix B for algorithm details. Each coloured line corresponds to a different line and direction in the tube's network.} \label{fig:london_kernel}
\end{figure}

\begin{figure}[h!]
	\centering
	\includegraphics[width=0.8\textwidth]{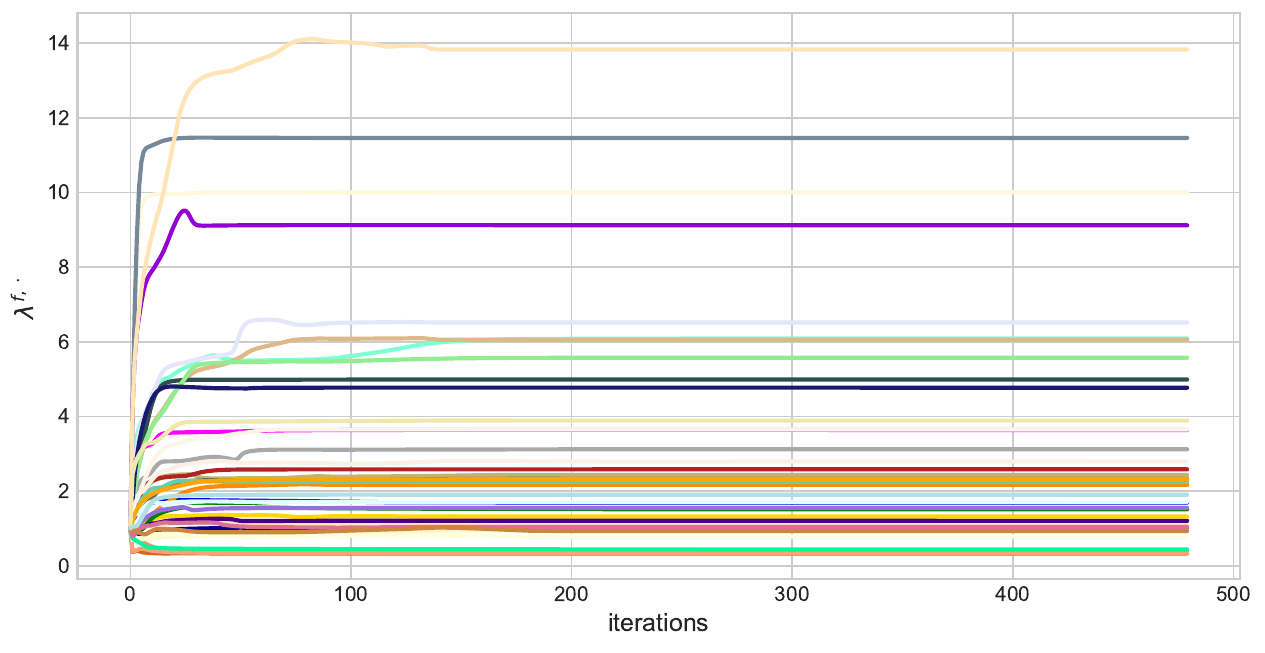}
	\caption{EM estimates using the Graph Filter-Smoother algorithm, of $\lambda^{f,\text{in}}$ and $\lambda^{f,\text{out}}$. Each coloured line correspond to a different station flow (inflow or outflow) in the tube's network.} \label{fig:london_lambda}
\end{figure}


The emission distribution is given by:
\begin{equation}\label{eq:emission_poi}
g(x,y) = \prod_{f=1}^{20} g^{f,\text{in}} \left ( x^{f,N(f)}, y^{f,\text{in}} \right ) g^{f,\text{out}} \left ( x^{N(f),f}, y^{f,\text{out}} \right ), \quad x \in \mathbb{X}^V,
\end{equation}
where $g^{f,\text{in}} \left ( x^{f,N(f)}, y^{f,\text{in}} \right )$ and $g^{f,\text{out}} \left ( x^{N(f),f}, y^{f,\text{out}} \right )$ are Poisson distributions:
\begin{align}
&g^{f,\text{in}}(x^{f,N(f)}, y^{f,\text{in}}) = \frac{ \left ( \lambda^{f,\text{in}} \sum_{j \in N(f)} x^{f,j} \right )^{y^{f,\text{in}}}}{(y^{f,\text{in}})!} e^{-(\lambda^{f,\text{in}} \sum_{j \in N(f)} x^{f,j})}\\
&g^{f,\text{out}}(x^{N(f),f}, y^{f,\text{out}}) =  \frac{ \left ( \lambda^{f,\text{out}} \sum_{i \in N(f)} x^{i,f} \right )^{y^{f,\text{out}}}}{(y^{f,\text{out}})!} e^{-(\lambda^{f,\text{out}} \sum_{i \in N(f)} x^{i,f})}
\end{align}
where $x_t^{f,N(f)}=(x_t^{f,j})_{j \in N(f)}$ and $x_t^{N(f),f}=(x_t^{i,f})_{i \in N(f)}$. Each pair of parameters $\lambda^{f,\text{in}}, \lambda^{f,\text{out}}$ has the interpretation as the overall intensity of inflow and outflow at the station corresponding to factor $f$ . The initial distribution and the transition matrix are written:
\begin{align}
&\mu_0(x) = \prod_{i = 1}^{20} \prod_{j \in N(i)} \mu_0^{i,j} \left ( x^{i,j} \right ), \quad x \in \mathbb{X}^V  \\ 
&p(x,z) = \prod_{i = 1}^{20} \prod_{j \in N(i)} p^{i,j} \left ( x^{i,j} , z^{i,j}\right ), \quad x,z \in \mathbb{X}^V.
\end{align}

For this model we have $M=48$ (the number of line segments between stations multiplied by two for the directions of travel on each line segment) and $L=4$, hence the cardinality of the overall state-space $\mathbb{X}^V$ is $4^{48}$. The factor graph has $\max\limits_{K \in \mathcal{K}}\mathbf{card}(N^m_v(K))=7$ and $\max\limits_{K \in \mathcal{K}}\mathbf{card}(N^m_f(K))=4$ hence the computational cost of the Graph Filter-Smoother, when $m=0$ and $\mathcal{K}=\{\{1\}, \dots, \{48\}\}$, is $\mathcal{O}[(48 \cdot 4)TL^{2 \cdot 7}]$, compared to $\mathcal{O}(48 TL^{48+1})$ for the FHMM forward-backward algorithm proposed by \cite{Ghahramani1997}. 

The total number of parameters of the model 1000: 40 from the flow intensity (2 for each of the 20 stations), 768 from the transition matrices, and 192 from the initial distribution. 

We now describe four sets of experiments. In the first set we demonstrate parameter estimation. In the second and third sets of experiments we illustrate predictive capabilities and comparison to a Long short-term Memory (LSTM) recurrent neural network, without and with missing data. In the fourth set of experiments we consider prediction in the presence of structural changes corresponding to hypothetical disruption and modification of the tube network.  

\paragraph{Parameter estimation.}

The parameters are estimated using an approximate EM algorithm, where the approximate smoothing distributions from the Graph Filter-Smoother are used in the E-step, more details are available in appendix B. We took $m=0$ and $\mathcal{K}$ as in subsection \ref{subsec:model_spec} (partition of singleton over $V$). The algorithm was run on the training set with random EM initializations for the estimated parameters. The estimates are shown in figure \ref{fig:london_kernel} and figure \ref{fig:london_lambda}. The estimates are found to be robust with respect to the variability in the EM initializations.

The estimate of the initial distribution for each $x_0^{i,j}$ puts very high probability on state $0$, which is to be expected since the tube lines are closed in the very early hours of the morning. 

The interpretation of the estimated transition probabilities is less straightforward, but we can discern interesting structure: upon inspecting the results we found that the state variables ${x_t^{i,j},i\neq j}$ could be roughly partitioned into four groups according to their estimated transition probabilities. The first group, which we shall heuristically refer to as \virg{stable}, are those for which the transition probabilities put most weight on maintaining a constant state, i.e. high probabilities on the diagonal of the transition matrix. Most state variables belong to this group. The second group, which we shall refer to as \virg{quiet}, consists of those variables for which there is little probability of transitioning to state $\{3\}$, which corresponds to the highest level of flow. State variables in the \virg{quiet} group typically correspond to line segments connecting hubs and low-flow stations (e.g Liverpool Street inflow with Aldgate outflow); high-flow stations and low-flow stations (e.g Canary Wharf inflow with North Greenwich outflow); low-flow stations and low-flow stations (e.g. Stepney Green inflow with Whitechapel outflow). We call the third group \virg{busy}, distinguished by transition probabilities in which transitions to states  $\{2,3\}$ are somewhat likely. These estimates are typical for lines connecting high-flow stations with high-flow stations, e.g. Liverpool Street inflow with Bank outflow. 


The interpretation of figure \ref{fig:london_lambda} is straightforward, indeed $\lambda^{f,\cdot}$ scales the intensity of the flow, hence bigger estimates of $\lambda^{f,\cdot}$ are referred to stations with higher flows.

\paragraph{Passenger flow prediction without missing data.}
Prediction over the test data is performed through the posterior predictive of Graph Filter-Smoother, i.e. per each time step $t\geq 0$ a posterior predictive sample $y_{t+1}$ is obtained as:
\begin{equation} \label{eq:posterior_pred}
x_t \sim \tilde{\pi}_{t|t}, \quad x_{t+1} \sim p(x_t,\cdot), \quad  y_{t+1} \sim g(x_{t+1}, \cdot) ,  
\end{equation}
where $g, p$ are computed through the EM algorithm run on the training set (Monday, Tuesday, Wednesday) and $\tilde{\pi}_{\cdot|\cdot}$ is obtained by running recursively the Graph Filter on the test data (Thursday, Friday). Note that \eqref{eq:posterior_pred} provides a sample from the posterior predictive, whose mean is then used as prediction for the test set. For ease of presentation we report the plots on four stations, plots on all the other stations are available in appendix C. The results indicate that the model is able to track the peaks of the inflow and outflow that occur during the morning and afternoon rush hours, which vary in magnitude from station to station. Moreover, credible intervals show satisfying coverage of the true data. 

\begin{figure}[h!]
	\centering 
	\includegraphics[clip,width = 0.9\textwidth]{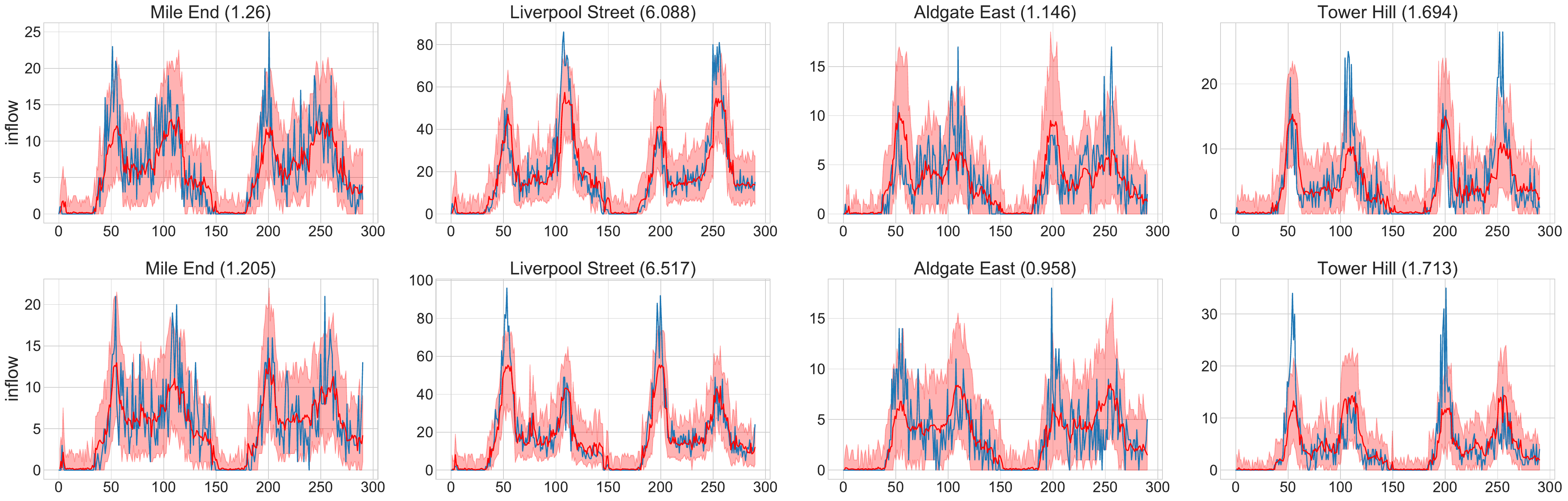}
	\caption{One step-ahead posterior predictive mean (solid red line) and 95\% credible intervals (red bands) using the Graph Filter-Smoother on four stations. Blue solid lines stand for the observed data from Thursday to Friday. The first row shows the inflow, the second row shows the outflow. The name of the station is reported at the top of each plot, along with the estimate of the corresponding $\lambda^{f,\cdot}$. } \label{fig:london_GFSprediction}
\end{figure}

\begin{figure}[h!]
	\centering
	\includegraphics[clip,width = 0.9\textwidth]{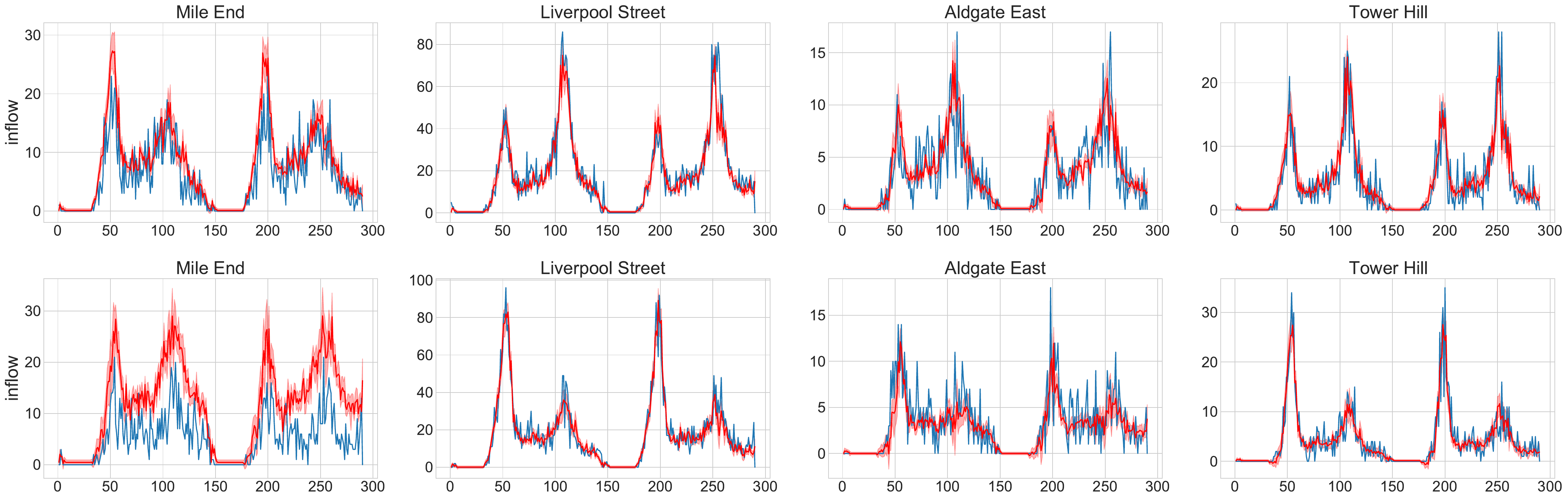}
	\caption{One step-ahead prediction with the LSTM. Per time step, a sample of size 100 is built over initializations and random number seeds used in training the  LSTM, where solid red lines show the mean and red bands show the region between the 0.025 and the 0.975 quantiles. Blue solid lines stand for the observed data from Thursday to Friday. The first row reports the inflow, the second row reports the outflow. The name of the station is written at the top of each plot.} \label{fig:london_LSTMprediction}
\end{figure}

We compare the proposed method with an LSTM trained on one-step-ahead prediction over the inflow-outflow, more details about the architecture and training are available in appendix C. The LSTM takes as input inflow-outflow data over all the stations at time $t$ and output predictions for time $t+1$. 100 LSTMs (different initializations and seeds) are trained on Monday, Tuesday and Wednesday. As for the Graph Filter-Smoother method, testing is performed on Thursday and Friday. We report in figure \ref{fig:london_LSTMprediction} the LSTMs predictions on four stations only, more details are available in the appendix C. An LSTM does not itself provide any uncertainty quantification associated with its predictions. We employ a commonly used heuristic of considering the variability of predictions across different initializations and random number seeds of the stochastic gradient algorithm used to train the LSTM. Figure \ref{fig:london_LSTMprediction} shows that different trainings of the LSTM lead to similar performances with bands that are narrower than figure \ref{fig:london_GFSprediction}.

\begin{table}[httb!]
	\centering
	\resizebox{0.9\columnwidth}{!}{
		\begin{tabular}{cccccccc}
			\hline
			\hline
			Method                                 & No missing        & Missing without peak  & Missing with peak  \\
			\hline
			\hline
			Graph Filter-Smoother                   & $4.796 \pm 0.011$ & $5.399\pm 0.014$     & $5.493 \pm 0.019$ \\
			\hline
			LSTM                                   & $4.639 \pm 0.138$ & $7.427 \pm 2.080$     & $6.178 \pm 1.425$ \\
			\hline
		\end{tabular}
	}
	\caption{RMSE comparison between the posterior predictive mean of Graph Filter-Smoother and LSTM. \virg{No missing} refers to the performance on the full test set. \virg{Missing without peak} refers to the test set performance when data from $t=130$ (around 9 pm on Thursday) to $t=170$ (around 4 am on Friday) are missing. \virg{Missing with peak} refers to the test set performance when data from $t=230$ (around 2 pm on Friday) to $t=270$ (around 9 pm on Friday) are missing. }
	\label{tab:GFSvsLSTM}
\end{table}

Table \ref{tab:GFSvsLSTM} reports RMSEs for the posterior predictive mean of Graph Filter-Smoother and the LSTM. The mean and standard deviation of the RMSE for the posterior predictive mean of Graph Filter-Smoother are computed over 100 samples of the posterior predictive mean. The mean and standard deviation of the RMSE for LSTM are computed over 100 LSTMs optimization (i.e. different initializations and seeds). The proposed algorithm has performances that are comparable with the LSTM and we note here that the LSTM has been explicitly trained to minimize the RMSE on the one-step-ahead prediction.

\begin{figure}[h!]
	\centering
	\includegraphics[clip,width=0.9\textwidth]{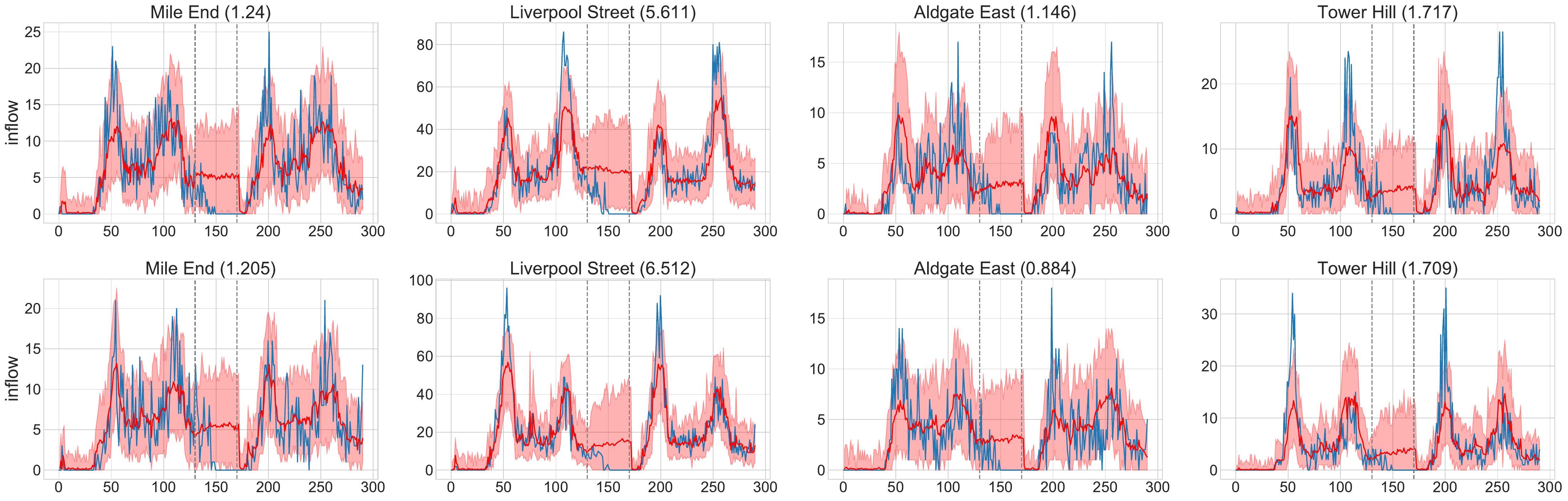}
	\caption{Multi-step-ahead posterior predictive mean (solid red line) and 95\% credible intervals (red bands) using the Graph Filter-Smoother on four stations with missing data in a quiet period (without peak). Blue solid lines stand for the observed data from Thursday to Friday (the missing data are included). Grey dashed lines show the start and the end of the missing data window.  The name of the station is reported at the top of each plot, along with the estimate of the corresponding $\lambda^{f,\cdot}$. } \label{fig:london_GFSpredictionnopeak}
\end{figure}

\begin{figure}[h!]
	\centering
	\includegraphics[clip,width=0.9\textwidth]{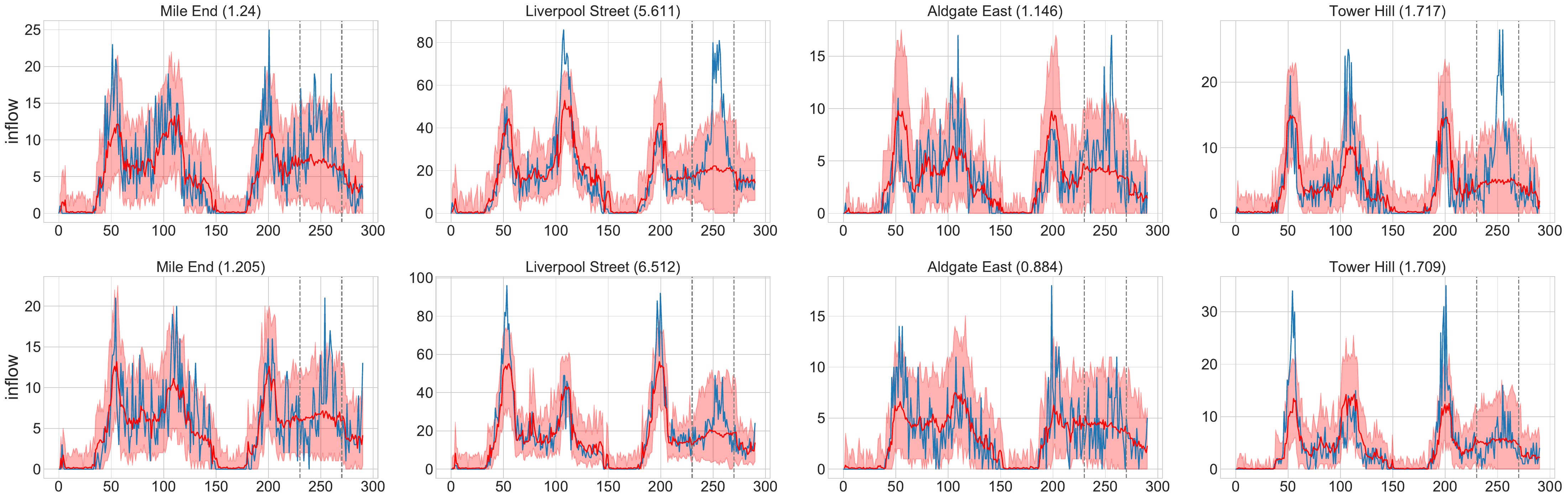}
	\caption{Multi-step-ahead posterior predictive mean (solid red line) and 95\% credible intervals (red bands) using the Graph Filter-Smoother on four stations with missing data in a busy period (with peak). Blue solid lines stand for the observed data from Thursday to Friday (the missing data are included). Grey dashed lines show the start and the end of the missing data window.  The name of the station is reported at the top of each plot, along with the estimate of the corresponding $\lambda^{f,\cdot}$.} \label{fig:london_GFSpredictionpeak}
\end{figure}

\paragraph{Passenger flow prediction with missing data.}

Suppose now that data cannot be collected in the period from $t$ to $t+h$. In practice this could correspond, for example, to hardware malfunctions at stations such that people come in and out without tapping their Oyster cards. The Graph Filter can easily be used to impute missing data by applying the \virg{prediction} operator without the \virg{correction} operator: we can sample from the posterior predictive from $t$ to $t+h$ as follows:

\begin{align}
&x_t \sim \tilde{\pi}_{t|t}, \\
&x_{t+1} \sim p(x_t,\cdot), \quad  \tilde{y}_{t+1} \sim g(x_{t+1}, \cdot) , \\
&x_{t+2} \sim p(x_{t+1},\cdot), \quad  \tilde{y}_{t+1} \sim g(x_{t+2}, \cdot) , \\
 &\dots \\
&x_{t+h} \sim p(x_{t+h-1},\cdot), \quad  \tilde{y}_{t+h} \sim g(x_{t+h}, \cdot).
\end{align}
Similarly, the LSTM can do multi-step-ahead prediction by using the output on the current time step as input for the next time step. We note that the considered LSTM is trained exclusively for one-step-ahead prediction (i.e. LSTM maps $y_t$ onto $y_{t+1}$). As an alternative, one could train LSTM on multi-step-ahead predictions (i.e. LSTM maps $y_t$ onto $y_{t+1}, \dots, y_{t+h}$), which requires to know the missing data window in advance. However, that is generally not plausible in practice, e.g. missing data in traffic applications are caused by broken sensors and the time they are out of order cannot be predicted in advance.
\begin{figure}[h!]
	\centering
	\includegraphics[clip,width=0.9\textwidth]{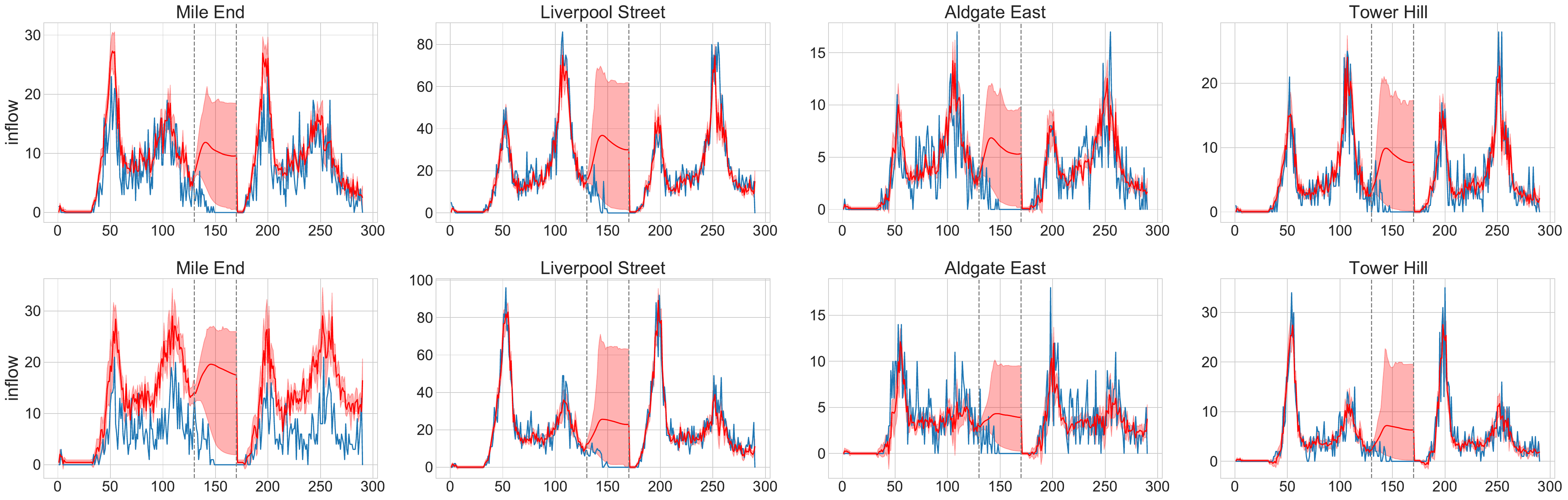}
	\caption{Multi-step-ahead prediction with the LSTM on four stations with missing data in a quiet period (without peak). Per time step, a sample of size 100 is built over different training of the LSTM where solid red lines show the mean and red bands show the region between the 0.025 and the 0.975 quantiles. Blue solid lines stand for the observed data from Thursday to Friday. Grey dashed lines show the start and the end of the missing data window. Stations' names are reported at the top of each plot.} \label{fig:london_LSTMpredictionnopeak}
\end{figure}

\begin{figure}[h!]
	\centering
	\includegraphics[clip,width=0.9\textwidth]{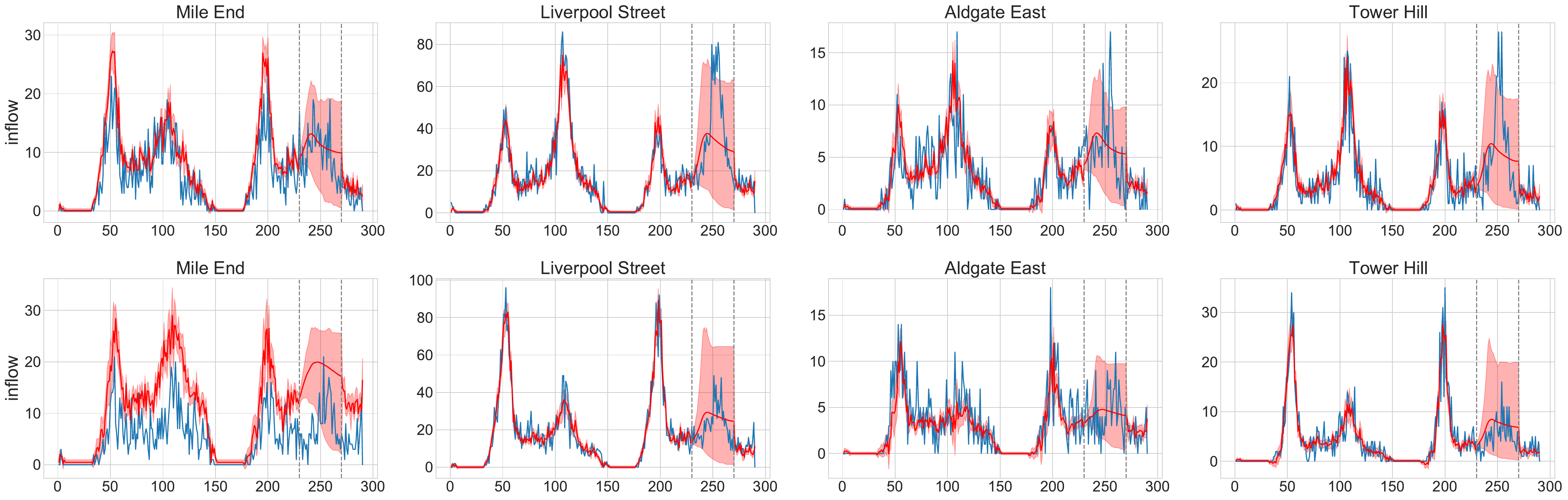}
	\caption{Multi-step-ahead prediction with the LSTM on four stations with missing data in a busy period (with peak). Per time step, a sample of size 100 is built over different training of the LSTM where solid red lines show the mean and red bands show the region between the 0.025 and the 0.975 quantiles. Blue solid lines stand for the observed data from Thursday to Friday. Grey dashed lines show the start and the end of the missing data window. Stations' names are reported at the top of each plot.} \label{fig:london_LSTMpredictionpeak}
\end{figure}

RMSE performances between LSTM and the proposed method in a missing data scenario are compared in the second and the third column of Table \ref{tab:GFSvsLSTM}. Two cases are distinguished: missing data in a quiet period (second column) and missing data in a busy period (third column). The LSTM performance is evidently not robust, with substantial variability in estimates corresponding to different initializations. On the contrary, the proposed method is more stable, with a standard deviation that is 100 times lower than LSTM and an RMSE that is significantly lower in mean. This appears in the experiments for missing data in both busy and quiet periods. Graphical illustrations on selected stations can be found in figures \ref{fig:london_GFSpredictionnopeak}, \ref{fig:london_GFSpredictionpeak}, \ref{fig:london_LSTMpredictionnopeak}, \ref{fig:london_LSTMpredictionpeak}, more figures can be found in appendix C.

\paragraph{Flow prediction under structural change - transport network disruptions and modifications.}
In this final set of experiments we demonstrate multi-step ahead prediction in a scenario where a change in the model occurs at a known point in time. We consider two types of change, which have the interpretation of i) hypothetical disruptions (or line segment closures) on the tube network, and ii) hypothetical modifications to the tube network by adding lines which do not currently exist. Making predictions in these scenarios is achievable because the hidden state variables have a straight-forward interpretation. This is in contrast, for example, to an LSTM, whose parameters are not interpretable in terms of the lines of the tube network and their respective levels of congestion.

For the case of a hypothetical disruption, the procedure is as follows. We first estimate model parameters using the Graph Filter-Smoother and EM as described above, using the data up to time point $200$. Beyond time step $200$ we then perform multi-step ahead prediction, again as described above, but with the state variable $x_t^{i,j}$ corresponding to the disrupted line conditioned to zero. For the case of a hypothetical addition of a line to the tube network, we again estimate model parameters up to time point $200$. For time steps beyond $200$ we introduce additional state variables $x_t^{i,j}$ corresponding the new lines and (for simplicity) set the prior Markov transition probability matrix for these individual state variables to the element-wise average of the corresponding matrices for all other lines in the network. We then perform multi-step ahead prediction for the augmented network.




Figure \ref{fig:modifiedlane} shows how the predicted passenger flows behave in two scenarios: one when a hypothetical disruption on the line between Bank and London Bridge is enforced, and one when the tube network is hypothetically modified by adding a line between Mile End and North Greenwich. These results illustrate that the disruption leads to decreased inflow-outflow at the stations at either end of the affected line segmented, while adding a line segment between two stations  results in increases inflow-outflow of the connected stations. 

\begin{figure}[h!]
	\begin{subfigure}[t]{0.5\textwidth}
		\includegraphics[clip,width=1\textwidth]{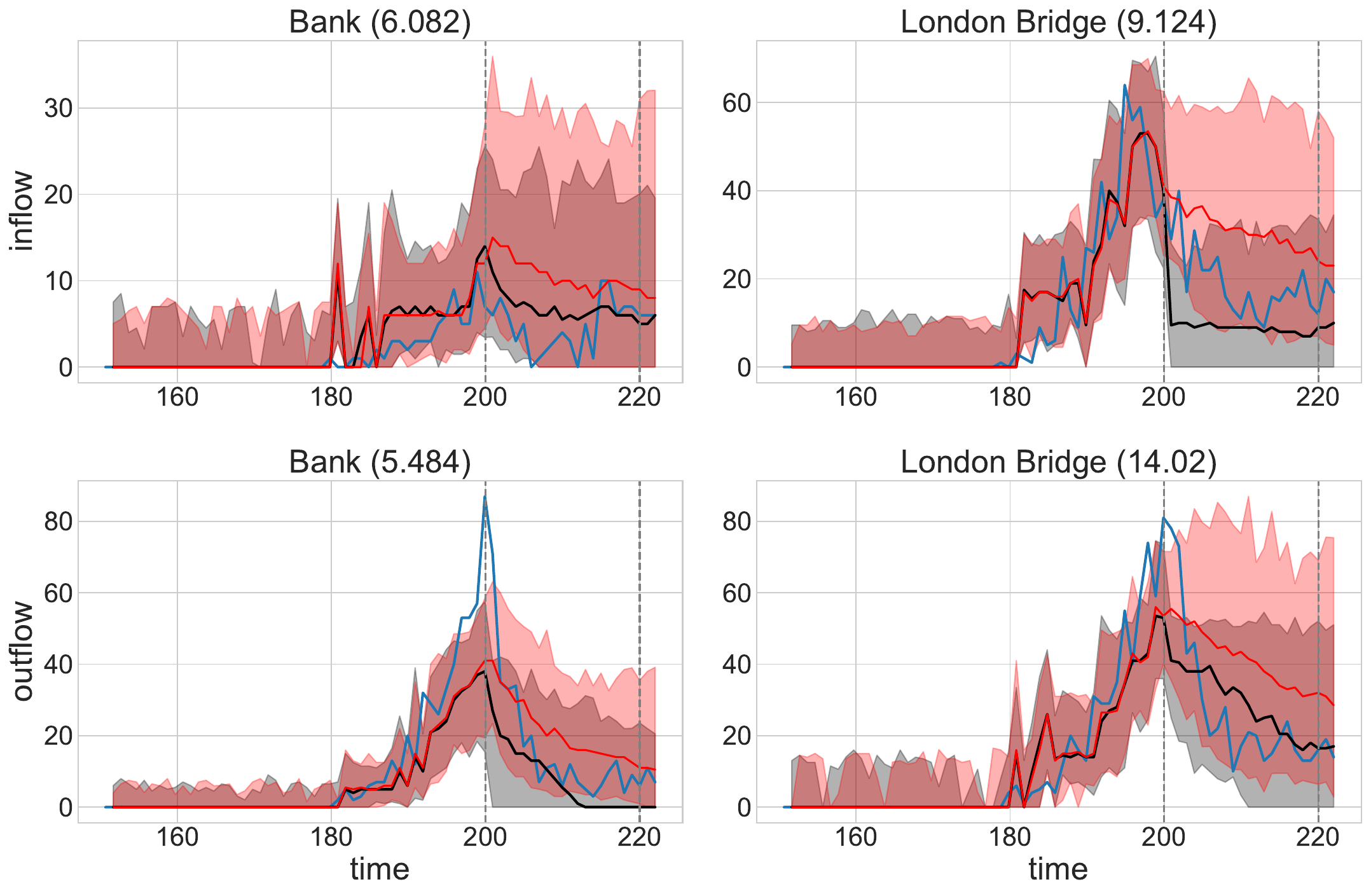}
	\end{subfigure}
	\begin{subfigure}[t]{0.5\textwidth}
		\includegraphics[clip,width=1\textwidth]{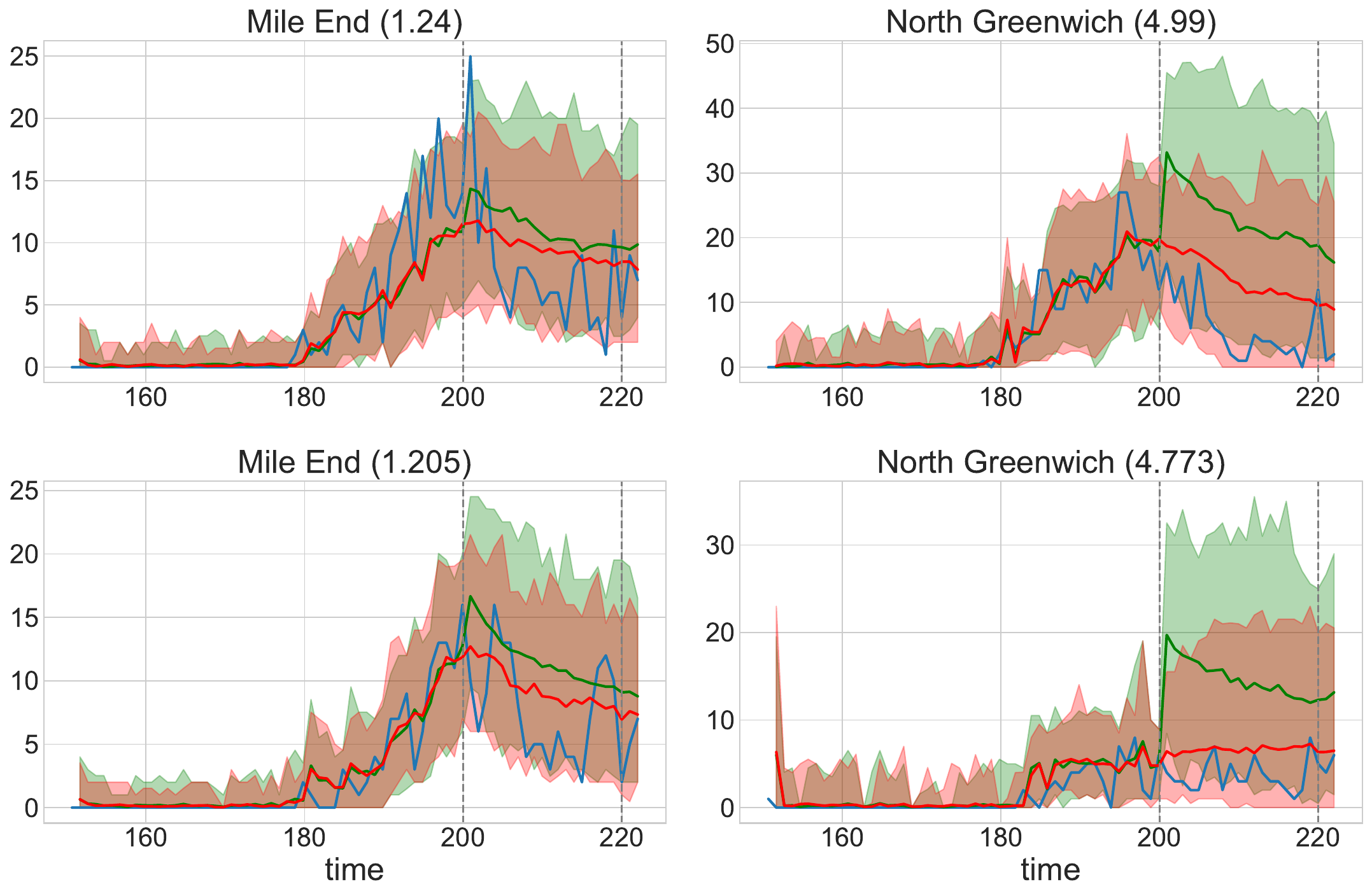}
	\end{subfigure}
	\caption{Predicted flows for Bank and London Bridge under a hypothetical disruption at time $200$ on the line between Bank and London Bridge (first and second columns), and for Mile End and North Greenwich when adding a line between these two stations at time $200$  (third and fourth columns). Top row: in-flow, bottom row: outflow. Black solid line and grey shading: posterior predictive mean and 95\% credible interval in the presence of the hypothetical disruption. Green solid line and green shading: posterior predictive mean and 95\% credible interval in the case when a hypothetical line is added. Red solid line and light pink shading: posterior predictive mean and 95\% credible interval without disruption or addition of a line. Darker pink/brown shading indicates intersection with respectively grey/green. Blue solid lines are the observed real data, in absence of disruption or addition of a line. Each subfigure title gives the station name and estimated $\lambda^{f,\cdot}$ value based on the data up to time $200$.} \label{fig:modifiedlane}
\end{figure}

\section{Conclusion}\label{sec:conc}
There are number of ways in which this work could be extended or generalized. From an algorithmic point of view, it seems natural to explore hybrids between variational methods of the sort proposed by \cite{Ghahramani1997} and the Graph Filter-Smoother. In the spirit of variational methods, could a layer of optimization in Kullback-Leibler divergence be somehow combined with the approximations in Graph Filter-Smoother, with the aim of yielding even more accurate approximations?

From a theoretical point of view, it would be interesting to investigate whether the kind of mathematical tools we have used to study the Graph Filter and Smoother could also help rigorously quantify the approximation error in variational methods for FHMMs, which appears to be an unsolved problem. Another interesting question beyond the scope of the present work is whether the generalized Dobrushin Comparison Theorems developed by \cite{rebeschini2014comparison} might help loosen some of the technical assumptions on which our results rely.

From a modelling point of view, there are many directions in which the London Underground example could be further investigated. The model we have considered is a simple prototype, with straightforward spacial interactions among lanes and stations. It would be interesting to see how it performs with spacial interactions over time, i.e. lanes do not evolve independently from each other, and with a richer state space. Even the emission distribution could be chosen more carefully, indeed a zero-inflated Poisson could be used to model more precisely the quiet periods.  


\section{Acknowledgments}
Nick Whiteley's research is supported by a fellowship from the Alan Turing Institute.


\vskip 0.2in

\bibliography{References}


\end{document}